\documentclass{article} 
\usepackage{arxiv_style,times}

\usepackage{amsmath,amsfonts,bm}

\def\eqref#1{equation~\ref{#1}}

\def\1{\bm{1}}

\DeclareMathAlphabet{\mathsfit}{\encodingdefault}{\sfdefault}{m}{sl}
\SetMathAlphabet{\mathsfit}{bold}{\encodingdefault}{\sfdefault}{bx}{n}

\usepackage{hyperref}
\usepackage{url}
\usepackage{bbm}
\usepackage{booktabs}
\usepackage{arydshln}
\usepackage{booktabs}
\usepackage{multirow}
\usepackage{graphicx}
\usepackage[normalem]{ulem}
\usepackage{mathtools}
\usepackage{wrapfig} 
\usepackage{paralist}
\usepackage{caption}
\usepackage{algorithm}
\usepackage{algpseudocode}

\newcommand{\ours}{ReQRL}

\newcommand{\best}[1]{{\boldmath
  \setbox0=\hbox{#1}%
  \rlap{\copy0}\raisebox{-1.8pt}[0pt][0pt]{\rule{\wd0}{0.4pt}}}}
\newcommand{\second}[1]{{\boldmath #1}}

\title{Learning Goal-Reaching Quasimetric \\Geometry from Finite-Time Reachability}

\author{Daisuke Yamada \& Travis Pence \\
Department of Computer Science\\
University of Wisconsin--Madison\\
\texttt{\{dyamada2, tnpence\}@wisc.edu} \\
\And
Vikas Singh \\
Department of Biostatistics and Medical Informatics \\
University of Wisconsin--Madison \\
\texttt{vsingh@biostat.wisc.edu} \\
}

\usepackage{amsthm}
\theoremstyle{plain}              
\newtheorem{theorem}{Theorem}[section]
\newtheorem{lemma}[theorem]{Lemma}

\theoremstyle{definition}         

\theoremstyle{remark}             

\iclrfinalcopy 
\begin{document}
\maketitle

\begin{abstract}
In goal-conditioned reinforcement learning (GCRL), quasimetric learning models goal-reaching costs as quasimetric distances, connecting local constraints to global value geometry. Its local constraints, however, should reflect the direction-dependent effects of control composition over a finite horizon together with environmental feasibility. We propose \emph{\ours{}}, which constrains the critic's value gradients through finite-horizon reachability. Drawing on state-constrained optimal control, we decouple dynamical reachability from boundary geometry, estimating both from data. On \texttt{OGBench}, our method outperforms or rivals existing quasimetric approaches and other offline GCRL methods.
\end{abstract}
\section{Introduction} \label{sec:intro}

Goal-conditioned reinforcement learning (GCRL) provides a framework for learning a shared policy across many goals, where experience collected while pursuing one goal can help the agent generalize to reach others. Specifying tasks through goals also reduces reliance on task-specific reward shaping \citep{schaul2015universal,andrychowicz2017hindsight}. A central challenge is learning optimal goal-conditioned values that describe the minimum cost of reaching each goal and guide policy learning and planning \citep{pong2018temporal,park2023hiql}. For instance, reaching a goal may require a long detour because an obstacle blocks the direct path or the available controls do not allow the agent to move directly toward it. Thus, the value function, or \emph{critic}, must reflect this global cost of reaching a goal, accounting for all the composed actions needed to follow the optimal feasible path. 

The goal-reaching costs above have a useful structure. They are nonnegative and zero when already at the goal, but may differ with the direction of travel. Moreover, the optimal cost of reaching the goal $g$ from a state $s$ cannot exceed the cost of going through an intermediate state $s'$. These properties precisely make the optimal goal-reaching costs $-V^\ast(s,g)$ a {\em quasimetric}. This distance structure has motivated various approaches \citep{wang2023optimal,myers2026offline,giammarino2026goal}, which restrict the critic's hypothesis class through quasimetric parameterization. In particular, to learn an environment's reaching costs, Quasimetric RL (QRL) \citep{wang2023optimal} imposes local constraints that keep predicted distances consistent with the costs of observed feasible transitions $(s,s')$.  On the other hand, EikHiQRL \citep{giammarino2026goal} directly constrains gradient magnitudes to regulate how rapidly its predicted distances change with state. The appeal of the quasimetric structure is that it links, or \emph{stitches}, reaching costs through intermediate states, so local supervision can inform global goal-reaching distance estimates.

Quasimetric parameterization links reaching costs across states, but does not specify which local changes in the state are physically possible or which paths are blocked. \emph{Local constraints} help convey this information to the critic, i.e., the learned value geometry depends on what these local constraints tell the critic about the environment's \emph{(i) control dynamics} (how actions change the state) and \emph{(ii) domain feasibility} (how the environment restricts feasible paths). Although QRL and EikHiQRL provide effective local supervision, their constraints remain limited. In practice, QRL enforces its local bounds on sparse sampled transitions in data. This leaves the critic to infer how its value should change along {\em other} feasible motions. EikHiQRL constrains local variation through a simplified isotropic dynamics model without an explicit boundary condition, which may impose a local distance geometry that does not match the agent's directional control capabilities or the domain's feasibility.

\setlength{\columnsep}{14pt}
\begin{wrapfigure}{r}{0.3\textwidth}
  \centering
  \vspace{-18pt}
  \includegraphics[width=\linewidth]{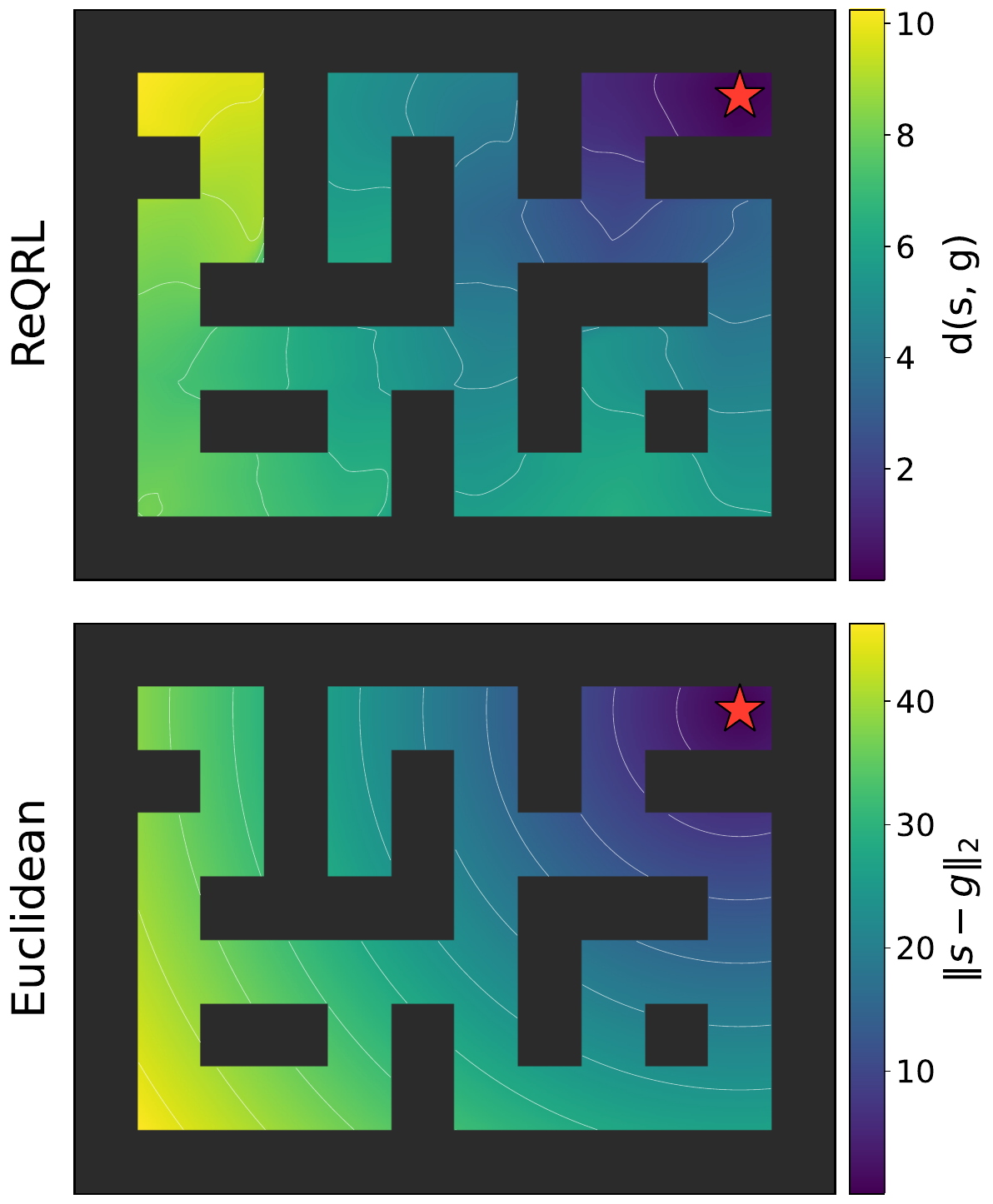}
  \captionsetup{font=footnotesize}
  \caption{\ours{}'s value geometry (top) and Euclidean (bottom). Colors for distance; dark regions denote obstacles, stars the goal.}
  \label{fig:distance-heat}
  \vspace{-18pt}
\end{wrapfigure} 

\textbf{This paper.} 
We ask: {\em can we provide local critic supervision that faithfully reflects both the system's control dynamics and domain feasibility?} To answer this, we show this supervision need {\em not} rely only on instantaneous control directions but can instead retain the effects of controls composed over time. Geometric control theory \citep{jean2014control} explains why: composing controls over a finite duration can produce state changes unavailable instantaneously. We capture these state changes through a \emph{finite-horizon reachable set}, and derive a \emph{finite-horizon Hamiltonian} encoding these effects in a gradient constraint by linearizing only the critic. To make reachability estimation tractable from transition data, we then draw on state-constrained Hamilton-Jacobi theory \citep{soner1986optimal} to decouple dynamical reachability from boundary geometry. We thus supply two complementary local training signals: (i) what composed controls can achieve dynamically and (ii) how domain boundaries restrict feasible progress toward a goal. Together they let the critic capture global goal-reaching geometry (Figure~\ref{fig:distance-heat}).

\vspace{3pt}
Our contributions: \textbf{(a)} we formulate quasimetric local constraints through a finite-horizon Hamiltonian retaining the effects of composed controls and decouple dynamical reachability from boundary feasibility using state-constrained Hamilton-Jacobi theory, \textbf{(b)} introduce \textbf{\ours{}}, a GCRL instantiation estimated from offline transition data, with closed-form zonotopic reachability constraints and boundary regularization, and \textbf{(c)} on \texttt{OGBench}, \ours{} largely outperforms QRL and EikHiQRL in success rate, matching or exceeding recent state-of-the-art GCRL methods on many tasks.

\section{Background}\label{sec:background}
\textbf{Goal-conditioned RL.}
Consider a deterministic goal-conditioned MDP $(\mathcal S, \mathcal A, \mathcal T, \rho_0, \rho_g, R, \gamma=1)$, where $\mathcal S \subseteq \mathbb R^n$ and $\mathcal A \subseteq \mathbb R^m$ are the state and action spaces. $\mathcal T\!:\mathcal S\times\mathcal A\to\mathcal S$ is the deterministic transition map, and $R(s,g)\coloneq -\mathbbm 1\{s\neq g\}$ is the sparse goal-conditioned reward. Let $\rho_0$ and $\rho_g$ denote the initial state and goal distributions. A goal-conditioned policy $\pi(\cdot\mid s,g)$ has value
\[
    V^\pi(s,g)\coloneq \mathbb E_{\pi}\!\left[\sum_{t=0}^{T-1} R(s_t,g) \,\middle|\, s_0=s\right],
\]
where $T$ is the first goal-hitting time. The policy aims to maximize $\mathbb E_{s_0\sim\rho_0,g\sim\rho_g}\left[V^{\pi}(s_0,g)\right]$, and the optimal value is $V^\ast(s,g)=\max_\pi V^\pi(s,g)$. In the offline setting, we learn from a fixed dataset $\mathcal D=\{\tau_i\}_{i=1}^N$ of trajectories $\tau_i=(s_0^i,a_0^i,s_1^i,\ldots,s_{T_i}^i)$ collected by an unknown behavior policy.

\textbf{Quasimetric RL.}
A \emph{quasimetric} is a function $d\!:\mathcal S\times\mathcal S\to\mathbb R_{\geq0}$ with $d(s,s)=0$ and the triangle inequality $d(s_1,s_3)\leq d(s_1,s_2)+d(s_2,s_3)$. It need not be symmetric. Let us denote by $\mathcal Q(\mathcal S)$ the class of all quasimetrics on $\mathcal S$. QRL \citep{wang2023optimal} observes that, under deterministic transitions in the undiscounted (i.e., $\gamma=1$) unit-cost goal-reaching setting, the negative optimal value $d^\ast(s,g):=-V^\ast(s,g)$ is the directed shortest-path distance from $s$ to $g$, and hence $d^\ast\in\mathcal Q(\mathcal S)$.  

QRL reframes GCRL as an optimization over $\mathcal Q(\mathcal S)$. They parameterize a critic as $d_\theta\in\mathcal Q(\mathcal S)$ using an Interval Quasimetric Embedding (IQE) \citep{wang2022improved}, whose architecture guarantees $d_\theta \in \mathcal Q(\mathcal S)$ for every $\theta$, and learn it by maximizing distances between $(s,g)$ pairs while constraining every observed transition by its local cost:
\begin{equation}
    \max_\theta\ \mathbb E_{s,g}[d_\theta(s,g)] \qquad \mathrm{s.t.}\qquad d_\theta(s,s')\leq 1,\quad \forall(s,a,s')\in\mathcal D.
    \label{eq:qrl}
\end{equation}

By the triangle inequality, these local constraints bound $d_\theta(s,g)$ by the cost of any path composed of constrained transitions. When optimizing over $\mathcal Q(\mathcal S)$ with sufficient transition and state--goal coverage, this objective makes the bound tight at the shortest path and recovers $d_{\theta^\ast}=-V^\ast$. The learned value can then be used for planning or policy learning, e.g., a DDPG-style actor \citep{lillicrap2015continuous} or advantage-weighted regression (AWR) \citep{peng2019advantage}.

EikQRL (Eikonal-QRL) \citep{giammarino2026goal} derives local constraints via a continuous-time limit. They assume deterministic Lipschitz dynamics admitting $\mathcal T(s,a)=s+f(s,a)\Delta t+o(\Delta t)$ with time step $\Delta t$ and velocity $f(s,a)$, unit running cost until reaching the goal, and a finite locally Lipschitz optimal value. Taking $\Delta t{\to}0$ yields the Hamilton-Jacobi-Bellman equation, which reduces under unit-speed isotropic dynamics, $f(s,a)=a$ with $\Vert a\Vert_2\leq1$, to the Eikonal equation $\Vert \nabla_{s} d^{\ast}\Vert_2 = 1$ (distance grows at unit rate). EikQRL imposes this softly with tolerance $\epsilon > 0$:
\begin{equation}
    \max_\theta\ \mathbb E_{s,g}[d_\theta(s,g)] \qquad \mathrm{s.t.}\qquad \mathbb E_{s,g}\left[ \left(\Vert\nabla_s d_\theta(s,g)\Vert_2-1\right)^2 \right]\leq\epsilon^2. 
    \label{eq:eikqrl}
\end{equation}
This gradient constraint is denser than QRL's, as it applies at all sampled $(s,g)$ pairs rather than along observed transitions in $\mathcal D$. However, the isotropic reduction of dynamics assumes an agent can move at unit speed equally in every direction at every state.  This also constrains the critic's local variation in the \emph{interior} but does not explicitly enforce boundary feasibility. In effect, the constraint acts as a dynamics-agnostic regularizer on the critic gradient in the domain interior. In the implementation, EikHiQRL combines this with HIQL's temporal abstraction \citep{park2023hiql}. 
\section{Going from Control Composition to Critic Supervision}\label{sec:method}

\setlength{\columnsep}{14pt}
\begin{wrapfigure}[23]{r}{0.3\textwidth}
  \centering
  \vspace{-5pt}
  \includegraphics[width=\linewidth]{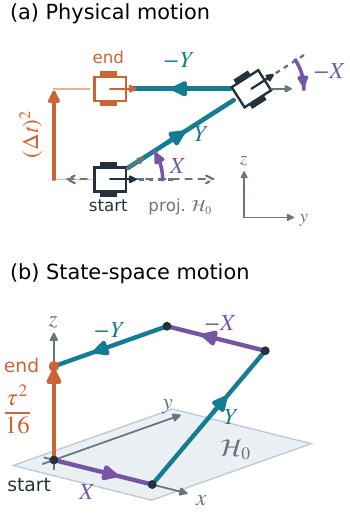}
  \vspace{-15pt}
  \captionsetup{font=footnotesize}
  \caption{Composed controls produce sideways motion (a), with net state change outside the initial instantaneous span $\mathcal H_0$ (b).}
  \label{fig:exp-ctrl-comp}
\end{wrapfigure}

Quasimetric structure determines local information propagation through intermediate states. We ask, \emph{what information about feasible motion should local constraints supply to the critic?} The answer depends on the time scale at which dynamics are viewed. We focus on deterministic control systems with $\mathcal S$ and $\mathcal A$. 

Instantaneous controls restrict motion to a few directions, but composing controls over a nonzero duration make additional directions reachable. We begin with a simple geometric control example.

\textbf{Example.}
Consider a simple planar motion controlled by turning and moving forward or backward with state $s=(x,y,z)$ representing its orientation angle, forward position, and sideways position. For small heading angles, a simplified model is
\[
    \dot x=a_1,\qquad \dot y=a_2,\qquad \dot z=xa_2, \qquad \Vert a\Vert_2\leq1,
\]
where $a_1$ controls turning and $a_2$ controls motion along the current heading. Thus, the control fields are $X(s)=(1,0,0)^\top$ and $Y(s)=(0,1,x)^\top$. At the origin, every instantaneous velocity has the form $(a_1,a_2,0) \in \mathcal H_0 := \operatorname{span}\{X(0),Y(0)\}$, so no control produces instantaneous sideways motion. Let us evolve the state by following $X,Y,-X,$ and $-Y$  each for duration $\Delta t>0$. After the duration $\tau = 4\Delta t$, the system reaches
\[
    (0,0,0)\xrightarrow{X}(\Delta t,0,0)
    \xrightarrow{Y}\left(\Delta t,\Delta t,(\Delta t)^2\right)
    \xrightarrow{-X}\left(0,\Delta t,(\Delta t)^2\right)
    \xrightarrow{-Y}\left(0,0,(\Delta t)^2\right).
\]
The sequence restores $x$ and $y$ while retaining a net change $(\Delta t)^2$ in $z$ (Figure~\ref{fig:exp-ctrl-comp}(a)). Therefore, control composition \emph{exposes} a reachable direction which was absent from the initial instantaneous control subspace $\mathcal H_0$ (Figure~\ref{fig:exp-ctrl-comp}(b)). This net change equals $(\Delta t)^2[X,Y](0)$, where the \emph{Lie bracket}
\[
    [X,Y](s):=DY(s)X(s)-DX(s)Y(s)=(0,0,1)^\top \notin \mathcal H_0
\]
characterizes this additional component where $DX$ and $DY$ are the field Jacobians \citep{agrachev2020comprehensive}. Over duration $\tau$, the state in $z$ evolves by $\tau^2/16$ with average velocity $(0,0,\tau/16)^\top$. Thus, the composed controls \emph{produce sideways motion at every finite $\tau > 0$}, but the average velocity vanishes as we take the limit $\tau \to 0$.

\emph{What does this illustrate?} Composing controls over a finite duration can produce motion outside the initial instantaneous control subspace \citep{jean2014control}. If we locally constrain the critic gradient $\nabla_s d_\theta$ only along control directions available instantaneously as $\tau \to  0$, the resulting condition stays insensitive to gradient components {\em orthogonal} to this subspace. However, a finite horizon $\tau > 0$ changes which components of the critic gradient a local constraint responds to. This motivates \emph{constraining the critic $d_\theta$ explicitly with information about motions achievable through control composition}.

\textbf{How does control composition shape goal-reaching values?}
Consider the undiscounted minimum-time problem with unit cost per unit time and a well-defined deterministic controlled evolution. For a supported finite duration $\tau>0$, let $\mathcal U_\tau(s)$ denote the \emph{feasible} controls of duration $\tau$ from $s$ using actions in $\mathcal A$. Each control determines a unique endpoint $\mathcal T_\tau(s,u)\in\mathcal S$. The standard Bellman optimality \citep{liberzon2011calculus} gives
\begin{equation}
    d^\ast(s,g) = \tau+\inf_{u\in\mathcal U_\tau(s)} d^\ast(\mathcal T_\tau(s,u),g), \qquad d^\ast(s,g)>\tau,
    \label{eq:tau-bellman}
\end{equation}
with $d^\ast(g,g)=0$. In (\ref{eq:tau-bellman}), the optimal cost is the elapsed duration $\tau$ plus the minimum remaining cost from a reachable endpoint. Notice the horizon $\tau$ determines how long controls are composed before evaluating the remaining cost. The endpoint map $u\mapsto\mathcal T_\tau(s,u)$ tells us which states can be reached from $s$ by composing controls over duration $\tau$ under \emph{both the environment's dynamics and its feasibility constraints}. These are exactly the states from which Bellman optimality evaluates the remaining goal-reaching cost, connecting composed controls to the value we aim to learn. Its image, translated by $-s$, gives the \emph{reachable set}
\begin{equation}
    \mathcal R_\tau(s)
    :=\{\mathcal T_\tau(s,u)-s:u\in\mathcal U_\tau(s)\}.
    \label{eq:true-reachability}
\end{equation}
Since the control enters (\ref{eq:tau-bellman}) through its endpoint, we can rewrite this standard condition as
\begin{equation}
    d^\ast(s,g) = \tau+\inf_{\delta\in\mathcal R_\tau(s)} d^\ast(s+\delta,g).
    \label{eq:tau-bellman-reachability}
\end{equation}
This relation requires the greatest descent in remaining cost over horizon $\tau$ to equal $\tau$. Our construction uses it to turn finite-horizon reachability into explicit supervision of the critic's local gradient.

\setlength{\columnsep}{14pt}
\begin{wrapfigure}[16]{r}{0.37\textwidth}
  \centering
  \vspace{-15pt}
  \includegraphics[width=\linewidth]{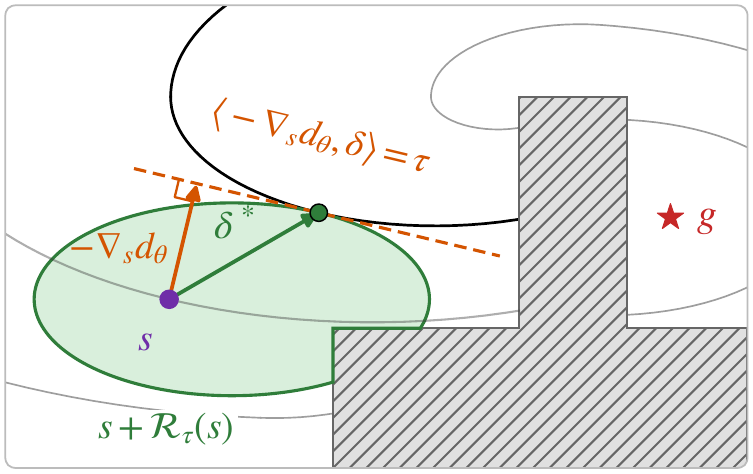}
  \vspace{-15pt}
  \captionsetup{font=footnotesize}
  \caption{Toy illustration of (\ref{eq:tau-hj}), with $d_\theta\approx d^*$ in value and gradient and $d^*$ affine along $[s,s+\delta^*]$. The orange line and black cost contour approximately touch the green reachable set at $s+\delta^*$. Contours are spaced by $\tau$.}
  \label{fig:reach-constraint}
\end{wrapfigure}

\textbf{Constraining the critic through a finite-horizon Hamiltonian.} 
With (\ref{eq:tau-bellman-reachability}) as our starting point, we seek to turn the relation into a tractable local training signal for the critic $d_\theta$. Unfortunately, enforcing it directly requires solving a nonlinear optimization to minimize the critic over reachable endpoints for each sampled $(s,g)$ pair \citep{rawlings2020model}. Instead, we simplify this inner optimization by \emph{linearizing only the critic's dependence on its endpoint} as 
\[
    d_\theta(s+\delta,g) = d_\theta(s,g) + \nabla_s d_\theta(s,g)^\top\delta + o(\Vert\delta\Vert).
\]
Crucially, this first-order approximation acts on the critic, while $\mathcal R_{\tau}(s)$ still contains the endpoints generated by composed controls. Thus, a gradient evaluated at a single state can receive supervision along directions that become reachable through control composition. At states where $d_\theta(\cdot,g)$ is differentiable, this approximation gives our \emph{finite-horizon Hamiltonian}
\begin{equation}
    H_\tau(s,\nabla_s d_\theta(s,g)) := \frac{1}{\tau} \sup_{\delta\in\mathcal R_\tau(s)} \langle-\nabla_s d_\theta(s,g),\delta\rangle-1. 
    \label{eq:finite-horizon-hamiltonian}
\end{equation}
As in classical minimum-time control, the Hamiltonian balances predicted cost descent against the running cost \citep{bardi1997optimal}. Similarly here (Figure~\ref{fig:reach-constraint}), matching the predicted descent to the elapsed cost in (\ref{eq:tau-bellman-reachability}) gives the first-order Bellman surrogate 
\begin{equation}
    H_\tau(s,\nabla_s d_\theta(s,g))=0 \quad\Longleftrightarrow\quad \sup_{\delta\in\mathcal R_\tau(s)} \langle-\nabla_s d_\theta(s,g),\delta\rangle=\tau.
    \label{eq:tau-hj}
\end{equation}
This expresses the descent condition in (\ref{eq:tau-bellman-reachability}) under the critic's first-order approximation. Going back to the example, the control sequence and its reverse at $s=0$ produce sideways motions $(0,0,\pm\tau^2/16)$ over duration $\tau$.  Since (\ref{eq:tau-hj}) requires the predicted descent along every reachable motion to be at most $\tau$, these motions imply
\[
    H_\tau(0,\nabla_s d_\theta(0,g)) + 1 \ge \frac{\tau}{16} |\partial_z d_\theta(0,g)| \implies |\partial_z d_\theta(0,g)| \le \frac{16}{\tau}.
\]
The critic's sideways gradient component thus enters the finite-horizon constraint, together with the other two components through $X$ and $Y$. Therefore, at $s=0$, \emph{all three components of $\nabla_s d_\theta$ enter the local constraint through motions reachable over a horizon $\tau>0$}, provided $d_\theta(\cdot,g)$ is differentiable there (our critic is almost everywhere differentiable). Through $\mathcal R_\tau(s)$, we can supply the critic with explicit information about both local control dynamics and feasibility when learning the global goal-reaching geometry. Notice, however, that $\mathcal R_\tau(s)$ combines two distinct sources of structure: (i) what the controls can generate and (ii) which of those motions actually remain feasible within the domain. In the next section, we separate their roles in the critic constraint.

\section{Decoupling: Dynamics in Interior, Feasibility at Boundary}
\label{sec:sec4}
Our formulation above combines two questions:
\begin{inparaenum}[\bfseries (i)]
    \item what state changes can the controls generate over duration $\tau$ and
    \item which trajectories remain feasible in $\mathcal S$?
\end{inparaenum}
Estimating this jointly from data requires accounting for how composed controls interact with environmental constraints over the full horizon \citep{tebjou2023data}. State-constrained Hamilton-Jacobi theory \citep{soner1986optimal,aubin2011viability} suggests a natural way to approach this difficulty: the ambient control dynamics give the interior Hamiltonian, and the geometry of the feasible domain shows up via a boundary viability condition. Following this structure, we approximate the finite-horizon reachability constraints using separate estimates of dynamical reachability for interior constraints and domain geometry for the boundary constraint. We will assume the underlying dynamics are defined independently of $\mathcal S$. 

\begin{figure}[t]
\centering
\includegraphics[width=0.9\linewidth]{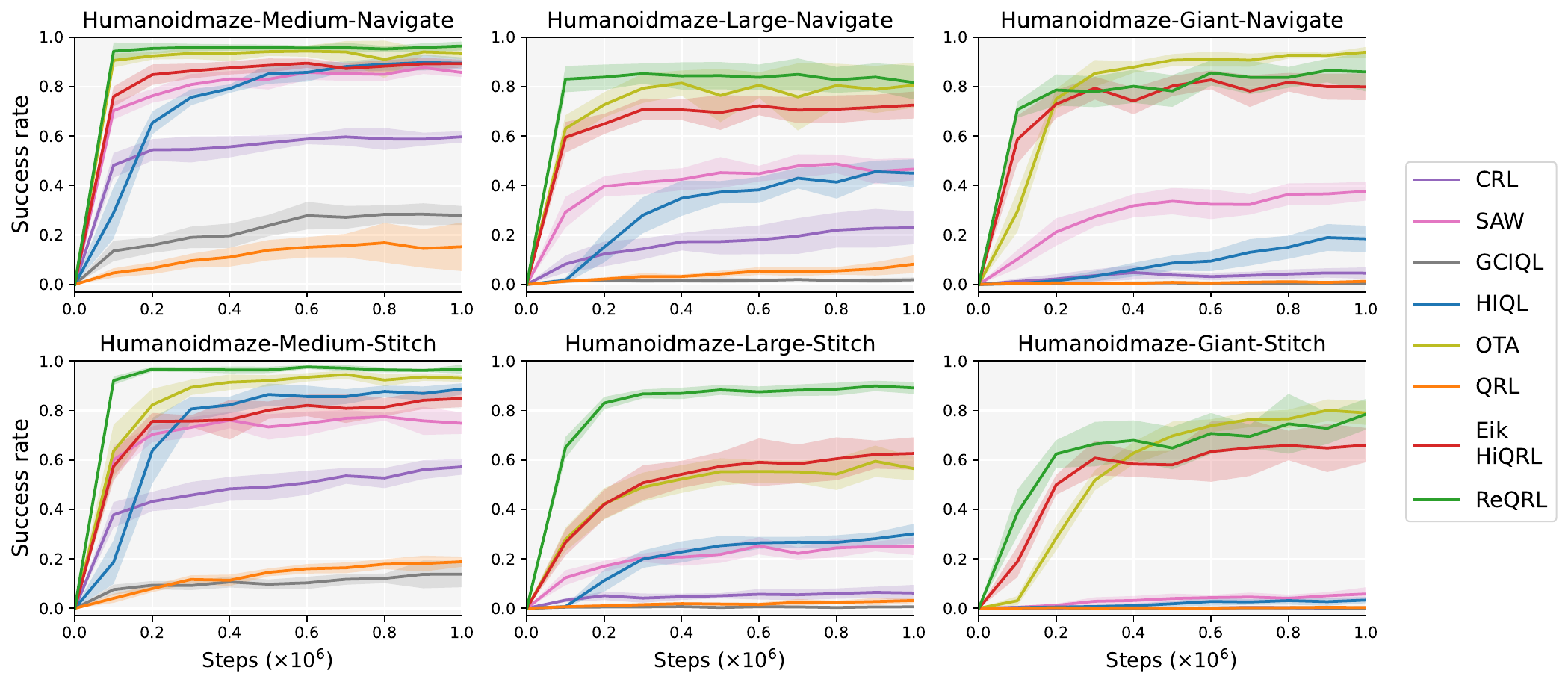}
\vspace{-7pt}
\caption{Evaluation curves on \texttt{humanoidmaze} for \texttt{navigate} and \texttt{stitch} datasets of all three sizes. The success rate with std shading is plotted every $10^5$  training steps up to $10^6$.}
\label{fig:humanoid-eval-curve}
\vspace{-6pt}
\end{figure}

\textbf{Setting up dynamical reachability.}
Let $\mathcal U_\tau^{\mathrm{dyn}}(s)$ denote the controls over duration $\tau$ using actions in $\mathcal A$ without requiring the trajectory to remain feasible in $\mathcal S$. Let us write $\mathcal T_\tau^{\mathrm{dyn}}(s,u)\in\mathbb R^n$ for its endpoint. Every feasible control is dynamically admissible, so $\mathcal U_\tau(s)\subseteq\mathcal U_\tau^{\mathrm{dyn}}(s)$. Under the stated assumption above, we have $\mathcal T_\tau^{\mathrm{dyn}}(s,u)=\mathcal T_\tau(s,u)$ for $u\in\mathcal U_\tau(s)$. Define
\begin{equation}
    \mathcal R_\tau^{\mathrm{dyn}}(s) :=\left\{\mathcal T_\tau^{\mathrm{dyn}}(s,u)-s: u\in\mathcal U_\tau^{\mathrm{dyn}}(s)\right\}.
    \label{eq:rk-dyn}
\end{equation}
This set describes what control dynamics permit over the horizon $\tau$ independently of the domain geometry, allowing the estimation of dynamical reachability and boundary feasibility separately. Assuming $\mathcal R_\tau^{\mathrm{dyn}}(s)$ is nonempty and compact, define the \emph{dynamical finite-horizon Hamiltonian} as
\[
    H_\tau^{\mathrm{dyn}}(s,p) :=\frac{1}{\tau} \sup_{\delta\in\mathcal R_\tau^{\mathrm{dyn}}(s)} \langle-p,\delta\rangle-1, \qquad p\in\mathbb R^n.
\]
In \citet{soner1986optimal}, the unconstrained dynamics determine the \emph{interior} equation, while the state constraint enters through the \emph{boundary}. Following the interior part of the state-constrained Hamilton-Jacobi structure, we impose the interior constraint $H_\tau^{\mathrm{dyn}}(s,\nabla_s d_\theta(s,g))=0$, or equivalently
\begin{equation}
    \sup_{\delta\in\mathcal R_\tau^{\mathrm{dyn}}(s)}
    \left\langle-\nabla_s d_\theta(s,g),\delta\right\rangle
    =\tau, \qquad s\in\mathcal S^\circ\setminus\{g\}.
    \label{eq:interior-const}
\end{equation}

\textbf{Setting up boundary feasibility.}
We incorporate domain geometry with a separate boundary condition. At a regular point $s\in\partial\mathcal S$, let $\nu(s)$ denote the outward unit normal. For solutions $C^1$ up to the boundary, the state-constraint condition of \citet{soner1986optimal} requires that adding any nonnegative outward-normal component to the covector cannot decrease the Hamiltonian. Adapting this gives $H_\tau^{\mathrm{dyn}}(s,p) \leq H_\tau^{\mathrm{dyn}}(s,p+\beta\nu(s))$ for all $\beta\geq0$ with $p = \nabla_s d_{\theta}(s,g)$. Assuming $\mathcal R_\tau^{\mathrm{dyn}}(s)$ is compact and its support maximizer $\delta^\ast$ is unique, convexity of the Hamiltonian in $p$ makes the above equivalent to
\begin{equation}
    \nu(s)^\top\delta^\ast\leq0, \qquad s \in \partial \mathcal S.
    \label{eq:anisotropic-boundary-simple}
\end{equation}
Thus, the reachable motion maximizing the critic’s predicted descent must be tangent to or point into the domain. For the isotropic case $\mathcal R_\tau^{\mathrm{dyn}}(s)=\tau\mathbb B_2$, with $\mathbb B_2$ the closed unit Euclidean ball, and $p\neq0$, we have $\delta^\ast=-\tau p/\Vert p\Vert_2$, so (\ref{eq:anisotropic-boundary-simple}) reduces to
\begin{equation}
    \nu(s)^\top\nabla_s d_\theta(s,g)\geq0.
    \label{eq:isotropic-boundary}
\end{equation}
\emph{Remark.} We account for feasibility through a local boundary condition, while (\ref{eq:tau-hj}) uses reachability that respects constraints over the full local horizon $\tau$. Under the stated assumptions, the constrained and unconstrained maxima agree whenever an unconstrained maximizing control remains feasible throughout that horizon. The approximation is practically motivated when trajectories remain close to their starting states relative to local geometric scales. We formalize the conditions for agreement and the possible gap in Appendix~\ref{app:proof}, and Section~\ref{sec:exp} shows its effectiveness in practice.

\textbf{Critic learning objective.}
Putting these two together, while we retain QRL's quasimetric parameterization and distance-maximization objective \citep{wang2023optimal}, we obtain our critic objective:
\begin{equation}
    \begin{aligned}
        \max_\theta \quad & \mathbb E_{s,g}[d_\theta(s,g)] \\
        \mathrm{s.t.}\quad
        & H_\tau^{\mathrm{dyn}}(s,\nabla_s d_\theta(s,g))=0,
        \quad s\in\mathcal S^\circ\setminus\{g\},\\
        & \nu(s)^\top\nabla_s d_\theta(s,g)\geq0,
        \quad s\in\partial\mathcal S \setminus\{g\}.
    \end{aligned}
    \label{eq:our-objective}
\end{equation}
We use the isotropic boundary form as a geometric regularizer discouraging outward Euclidean descent without estimating a maximizing motion at the boundary. This simplifies boundary enforcement while retaining anisotropic reachability in the interior. In practice, we use soft constraint penalties and QRL's monotone distance transform (Appendix~\ref{app:exp}). Section~\ref{sec:offline-inst} describes our offline instantiation and estimation of dynamical reachability and boundary geometry.
\section{From Transition Data to Critic Constraints}
\label{sec:offline-inst}
In deterministic GCRL, the dataset $\mathcal D$ contains transitions $(s, a, s')$ with $s' = \mathcal T(s,a)$ at a fixed, unknown interaction interval $\Delta t > 0$. To work directly with this data, we instantiate the horizon as $\tau := k\Delta t$ with $k\in\mathbb N$. Controls over this duration are action sequences $u=(a_0,\ldots,a_{k-1})\in\mathcal A^k$, with endpoints obtained by composing $\mathcal T$ for $k$ environment steps. This uses the transition map directly without assuming a continuous-time generator, which need not exist or be unique \citep{palis1974vector,yue2016inverse}. This $k$-step composition is related to temporal abstraction \citep{sutton1999between}; here, the composition defines the reachable set used in the finite-horizon critic constraint.

\textbf{Reachability and cost normalization.}
At $\tau=k\Delta t$, we write $\mathcal R_k^{\mathrm{dyn}}(s):=\mathcal R_{k\Delta t}^{\mathrm{dyn}}(s)$ for the $k$-step reachable sets. We measure goal-reaching costs in units of the fixed horizon $\tau$, dividing every policy's accumulated time cost by the same positive constant, preserving the set of optimal policies exactly. With $d_\theta$ denoting the normalized critic, the objective in (\ref{eq:our-objective}) has the equivalent form
\begin{equation}
    \begin{aligned}
        \max_\theta \quad & \mathbb E_{s,g}[d_\theta(s,g)] \\
        \mathrm{s.t.}\quad
        & \max_{\delta\in\mathcal R_k^{\mathrm{dyn}}(s)}
        \langle-\nabla_s d_\theta(s,g),\delta\rangle=1,
        \quad s\in\mathcal S^\circ\setminus\{g\},\\
        & \nu(s)^\top\nabla_s d_\theta(s,g)\geq0,
        \quad s\in\partial\mathcal S \setminus\{g\}.
    \end{aligned}
    \label{eq:normalized-objective}
\end{equation}
Positive scaling preserves the quasimetric structure and the boundary condition, and this formulation can be evaluated even when the numerical value of $\Delta t$ is unknown. In our instantiation, we parameterize the critic using goal-relevant coordinates $x_s\in\mathbb R^{n_x}$ chosen offline, writing $d_\theta(s,g)=d_\theta(x_s,x_g)$. This restricts the approximation of the full-state value. To estimate reachable motions in these coordinates, we use a richer reduced state $r_s\in\mathbb R^{n_r}$ containing $x_s$ and relevant dynamical variables, and then project the predicted motions onto the critic coordinates. We estimate dynamical reachability and boundary geometry from $\mathcal D$ using an affine dynamics model and empirical state occupancy, respectively. See Appendix~\ref{app:exp} for full details.

\textbf{Estimating dynamical reachability.}
We fit an affine model of the dynamics by ordinary least squares on offline transitions $(s, a, s') \in \mathcal D$ \citep{ljung1999system,verhaegen2007filtering}: $r_{s'}=A_r r_s+B_r a+b_r+w$, with $a\in\mathcal A=[-1,1]^m$, where $w$ represents affine-model discrepancies, including nonlinear effects and omitted state dependence, and ranges over an empirical zonotopic error envelope estimated from residuals. Propagating this model for $k$ steps and retaining the goal-coordinate displacements \citep{girard2005reachability,althoff2021set} yields the estimate
\begin{equation}
    \widehat{\mathcal R}_{k}^{\mathrm{dyn}}(s) =\{c_s+G_k z:\Vert z\Vert_\infty\leq1\}.
    \label{eq:estimated-rk}
\end{equation}
We obtain $c_s$ and $G_k$ by initializing $c=r_s$ with no generators and iterating $c\leftarrow A_r c+b_r$ and $G\leftarrow[A_rG,B_r,G_w^\top]$ for $k$ steps, where the error envelope is $\{G_w^\top\xi:\Vert\xi\Vert_\infty\leq1\}$. We retain the goal-coordinate entries and rows and subtract $x_s$ from the center to obtain displacements. $z$ stacks the action and error coefficients across these steps. This propagation captures multistep directions expressible through the fitted affine model, but is an approximation to the general finite-horizon reachable set. The resulting set is a \emph{zonotope}, whose support function has a closed form \citep{althoff2016combining,althoff2021set}. Thus, the estimated interior constraint in (\ref{eq:normalized-objective}) becomes
\begin{equation}
    -\langle \nabla_{x_s}d_\theta(x_s,x_g),c_s\rangle +\Vert G_k^\top\nabla_{x_s}d_\theta(x_s,x_g)\Vert_1=1.
    \label{eq:zonotope-support}
\end{equation}
We can evaluate this during training without enumerating action sequences or reachable-set vertices.

\begin{table}[t]
\centering
\footnotesize
\setlength{\tabcolsep}{4.5pt}
\renewcommand{\arraystretch}{1.1}
 
\scalebox{.82}{\begin{tabular}{ll l llllll @{\hspace{10pt}} lll}
\toprule
\multirow{2}{*}{\textbf{Env}} & \multirow{2}{*}{\textbf{Data}} & \multirow{2}{*}{\textbf{Size}}
& \multicolumn{6}{c}{\textbf{Other GCRL}} & \multicolumn{3}{c}{\textbf{Quasimetric}} \\
\cmidrule(lr){4-9}\cmidrule(lr){10-12}
& & & \multicolumn{1}{c}{CRL} & \multicolumn{1}{c}{SAW} & \multicolumn{1}{c}{GCIQL} & \multicolumn{1}{c}{HIQL} & \multicolumn{1}{c}{OTA} & \multicolumn{1}{c}{GAS}
& \multicolumn{1}{c}{QRL} & \multicolumn{1}{c}{EikHiQRL} & \multicolumn{1}{c}{\textbf{\ours}} \\
\midrule

\multirow{6}{*}{\rotatebox[origin=c]{90}{\texttt{pointmaze}}}
& \multirow{3}{*}{\texttt{nav}}
& \texttt{medium}   & $65 \pm 10$ & \best{$99 \pm 2$} & $57 \pm 1$ & $92 \pm 3$  & \second{$94 \pm 3$} & $75 \pm 12$         & $79 \pm 9$ & $82 \pm 9$  & \second{$94 \pm 7$} \\
& & \texttt{large}  & $50 \pm 12$ & $60 \pm 14$       & $33 \pm 2$ & $49 \pm 12$ & \second{$97 \pm 3$} & $72 \pm 9$          & $88 \pm 9$ & $83 \pm 13$ & \best{$98 \pm 3$} \\
& & \texttt{giant}  & $26 \pm 10$ & $12 \pm 14$       & $0 \pm 0$  & $6 \pm 7$   & $83 \pm 11$         & \second{$86 \pm 8$} & $73 \pm 7$ & $71 \pm 18$ & \best{$91 \pm 8$} \\
\cmidrule(l){2-12}
& \multirow{3}{*}{\texttt{sti}}
& \texttt{medium}   & $2 \pm 3$ & $70 \pm 14$ & $30 \pm 5$ & $68 \pm 12$ & $81 \pm 7$  & $85 \pm 11$ & $80 \pm 8$  & $87 \pm 12$ & \best{$93 \pm 10$} \\
& & \texttt{large}  & $0 \pm 0$ & $12 \pm 10$ & $27 \pm 7$ & $12 \pm 8$  & $37 \pm 14$ & $87 \pm 13$ & $82 \pm 8$  & $81 \pm 11$ & \best{$96 \pm 8$} \\
& & \texttt{giant}  & $0 \pm 0$ & $2 \pm 6$   & $0 \pm 0$  & $0 \pm 0$   & $9 \pm 10$  & $71 \pm 22$ & $59 \pm 10$ & $60 \pm 20$ & \best{$81 \pm 15$} \\
\midrule

\multirow{8}{*}{\rotatebox[origin=c]{90}{\texttt{antmaze}}}
& \multirow{3}{*}{\texttt{nav}}
& \texttt{medium}   & $76 \pm 5$ & \second{$93 \pm 3$} & $58 \pm 9$ & \second{$91 \pm 3$} & \second{$91 \pm 1$} & \best{$96 \pm 1$}   & $83 \pm 9$  & \second{$95 \pm 2$} & \second{$94 \pm 5$} \\
& & \texttt{large}  & $58 \pm 7$ & $73 \pm 2$ & $20 \pm 4$ & $72 \pm 4$          & \second{$91 \pm 2$} & \second{$92 \pm 3$} & $65 \pm 18$ & $83 \pm 3$          & \best{$95 \pm 1$} \\
& & \texttt{giant}  & $5 \pm 3$  & $11 \pm 4$ & $0 \pm 0$  & $7 \pm 4$           & $35 \pm 3$          & $47 \pm 25$         & $0 \pm 0$   & $56 \pm 9$          & \best{$77 \pm 6$} \\
\cmidrule(l){2-12}
& \multirow{3}{*}{\texttt{sti}}
& \texttt{medium}   & $33 \pm 5$ & $72 \pm 3$ & $38 \pm 4$ & $74 \pm 5$ & \second{$95 \pm 2$} & \best{$97 \pm 1$} & $68 \pm 8$ & \second{$94 \pm 2$}  & \second{$96 \pm 1$} \\
& & \texttt{large}  & $7 \pm 4$  & $23 \pm 5$ & $6 \pm 3$  & $15 \pm 4$ & $59 \pm 11$         & \best{$92 \pm 2$} & $13 \pm 3$ & $81 \pm 6$           & \second{$89 \pm 3$} \\
& & \texttt{giant}  & $0 \pm 0$  & $0 \pm 0$  & $0 \pm 0$  & $0 \pm 0$  & $3 \pm 1$           & \best{$72 \pm 6$} & $0 \pm 0$  & $41 \pm 16$          & \second{$67 \pm 6$} \\
\cmidrule(l){2-12}
& \multirow{2}{*}{\texttt{exp}}
  & \texttt{medium} & $4 \pm 4$ & $31 \pm 9$ & $12 \pm 4$ & $30 \pm 14$ & $90 \pm 12$ & \second{$96 \pm 3$} & $1 \pm 2$ & $45 \pm 21$ & \best{$97 \pm 4$}\\
& & \texttt{large}  & $0 \pm 0$ & $4 \pm 2$  & $0 \pm 1$  & $1 \pm  2$  & $79 \pm 12$ & \best{$92 \pm 4$} & $0 \pm 0$ & $9 \pm 9$   & $83 \pm 6$\\
\midrule

\multirow{6}{*}{\rotatebox[origin=c]{90}{\texttt{humanoid}}}
& \multirow{3}{*}{\texttt{nav}}
& \texttt{medium}   & $60 \pm 2$ & $86 \pm 2$ & $28\pm 4$ & $89 \pm 2$ & \second{$94 \pm 3$} & \second{$94 \pm 6$} & $15 \pm 10$ & $89 \pm 2$ & \best{$96 \pm 2$} \\
& & \texttt{large}  & $23 \pm 7$ & $47 \pm 5$ & $2\pm 1$  & $45 \pm 6$ & $81 \pm 9$          & \best{$87 \pm 5$}   & $8 \pm 4$   & $73 \pm 5$ & \second{$82 \pm 7$} \\
& & \texttt{giant}  & $5 \pm 2$  & $38 \pm 4$ & $0\pm 1$  & $18 \pm 5$ & \best{$95 \pm 2$}   & $86 \pm 3$          & $1 \pm 1$   & $80 \pm 5$ & $86 \pm 8$ \\
\cmidrule(l){2-12}
& \multirow{3}{*}{\texttt{sti}}
  & \texttt{medium} & $57 \pm 3$ & $75 \pm 4$ & $14 \pm 5$ & $89 \pm 2$ & \second{$93 \pm 1$} & \best{$98 \pm 2$}   & $19 \pm 2$ & $85 \pm 4$ & \second{$97 \pm 2$} \\
& & \texttt{large}  & $6 \pm 3$  & $25 \pm 4$ & $1 \pm 0$  & $30 \pm 4$ & $56 \pm 5$          & $82 \pm 5$          & $3 \pm 1$  & $63 \pm 7$ & \best{$89 \pm 2$} \\
& & \texttt{giant}  & $0 \pm 0$  & $6 \pm 3$  & $0 \pm 0$  & $3 \pm 1$  & \second{$79 \pm 5$} & \best{$83 \pm 3$}   & $0 \pm 0$  & $66 \pm 7$ & \second{$79 \pm 6$} \\

\bottomrule
\end{tabular}}
\vspace{-1pt}
\caption{\textbf{Success rate \% ($\uparrow$)}, percentage of episodes in which agent reaches goal $g$ from the starting point $s$ (mean $\pm$ std over 10 seeds of the final evaluation). The agent is tested for 50 episodes per goal. We \textbf{bold} the results within $95\%$ of the best and \underline{underline} the best.}
\vspace{-12pt}
\label{tab:main}
\end{table}

\textbf{Estimating boundary geometry.}
We estimate goal-coordinate boundary geometry by rasterizing offline observations of $x_s$ and applying morphological closing \citep{haralick1987image}. From this occupancy region, we construct a smooth implicit representation $\Phi$ of the estimated boundary, positive inside and negative outside, which gives $\widehat\nu(x_s)=-\nabla\Phi(x_s)/\Vert\nabla\Phi(x_s)\Vert_2$. This follows the level-set characterization of boundary normals \citep{delfour2011shapes}. We enforce the boundary penalty near the estimated boundary, assuming adequate coverage of the feasible region.

\textbf{Finite time scale $k$.}
The horizon $k \ge 1$ sets the duration $\tau=k\Delta t$ over which controls are composed. We choose $k$ by comparing the extent of the control-generated part of $G_k$ with the nominal center displacement $c_s$ in the goal coordinates. The control-generated part comes from the propagated generators $B_r,A_rB_r,\ldots,A_r^{k-1}B_r$ \citep{chen1984linear,althoff2021set}. We choose a horizon at which these two contributions are comparable, so that composed actions can affect the goal coordinates. We fix $k$ within each environment family across dataset types and domain sizes. 
\section{Experiments}\label{sec:exp}

\textbf{Policy learning and control.}
We train an IQE critic using soft penalties for the normalized objective in (\ref{eq:normalized-objective}) \citep{wang2022improved}. The reachable set and boundary normals follow Section~\ref{sec:offline-inst}, with the interior penalty based on the zonotopic support in (\ref{eq:zonotope-support}). Our critic receives \emph{local training signals from motions composed over $\tau=k\Delta t$}. Thus, we supervise an actor at the same horizon, training an actor $\pi_\psi(w\mid x_s,x_g)$ on $k$-step hindsight \emph{waypoints} $w=x_{s_{t+k}}$, truncated at trajectory ends. This aligns waypoint targets with the critic's reachability horizon and learned geometry. The actor predicts an offset from $x_s$ and is trained by capped AWR \citep{peng2019advantage} following HIQL \citep{park2023hiql}, using advantage $A^h(s,w,g) =d_\theta(x_s,x_g)-d_\theta(w,x_g)$. We supply these coordinate waypoints directly to the low-level policy, without any learned waypoint embedding. For $k>1$, a separate value function $V_\eta(s,w)$ takes the full observation and waypoint and is trained by action-free, goal-conditioned expectile TD learning \citep{kostrikov2021offline}. The low-level policy $\pi_\ell(a\mid s,w)$ is then trained by AWR on recorded actions using HIQL's simplified advantage $V_\eta(s',w)-V_\eta(s,w)$. Full training and implementation details appear in Appendix~\ref{app:exp}.

\subsection{Main Benchmarks} 
\begin{wraptable}{r}{0.5\textwidth}
\centering
\footnotesize
\setlength{\tabcolsep}{3pt}
\renewcommand{\arraystretch}{1.1}

\vspace{-40pt}
\scalebox{0.75}{

\begin{tabular}{ll l ll @{\hspace{10pt}} lll}
\toprule
\multirow{2}{*}{\textbf{Env}} & \multirow{2}{*}{\textbf{Data}} & \multirow{2}{*}{\textbf{Size}}
& \multicolumn{2}{c}{\textbf{Other GCRL}} & \multicolumn{3}{c}{\textbf{Quasimetric}} \\
\cmidrule(lr){4-5}\cmidrule(lr){6-8}
& & & \multicolumn{1}{c}{OTA} & \multicolumn{1}{c}{GAS}
& \multicolumn{1}{c}{QRL} & \multicolumn{1}{c}{EikHiQRL} & \multicolumn{1}{c}{\textbf{\ours}} \\
\midrule

\multirow{8}{*}{\rotatebox[origin=c]{90}{\texttt{antmaze}}}
& \multirow{3}{*}{\texttt{nav}}
& \texttt{medium}  & $3 \pm 0$  & $2 \pm 0$ & $5 \pm 1$  & $3 \pm 0$ & $1 \pm 0$ \\
& & \texttt{large} & $3 \pm 0$  & $3 \pm 1$ & $7 \pm 2$  & $5 \pm 0$ & $2 \pm 0$ \\
& & \texttt{giant} & $9 \pm 1$  & $8 \pm 3$ & $14 \pm 1$ & $6 \pm 1$ & $4 \pm 1$ \\
\cmidrule(l){2-8}
& \multirow{3}{*}{\texttt{sti}}
& \texttt{medium}  & $3 \pm 0$  & $2 \pm 0$ & $6 \pm 1$  & $4 \pm 0$ & $1 \pm 0$ \\
& & \texttt{large} & $8 \pm 1$  & $3 \pm 1$ & $10 \pm 1$ & $5 \pm 1$ & $1 \pm 0$ \\
& & \texttt{giant} & $13 \pm 1$ & $5 \pm 1$ & $16 \pm 1$ & $7 \pm 1$ & $4 \pm 1$ \\
\cmidrule(l){2-8}
& \multirow{2}{*}{\texttt{exp}}
& \texttt{medium}  & $2 \pm 1$ & $4 \pm 1$ & $21 \pm 2$ & $13 \pm 3$ & $1 \pm 1$ \\
& & \texttt{large} & $5 \pm 2$ & $5 \pm 1$ & $19 \pm 1$ & $11 \pm 2$ & $1 \pm 1$ \\
\midrule

\multirow{6}{*}{\rotatebox[origin=c]{90}{\texttt{humanoid}}}
& \multirow{3}{*}{\texttt{nav}}
& \texttt{medium}  & $8 \pm 1$ & $7 \pm 2$  & $16 \pm 4$ & $9 \pm 1$  & $3 \pm 0$ \\
& & \texttt{large} & $9 \pm 1$ & $9 \pm 1$  & $18 \pm 3$ & $13 \pm 2$ & $10 \pm 3$ \\
& & \texttt{giant} & $9 \pm 1$ & $11 \pm 1$ & $25 \pm 2$ & $11 \pm 1$ & $7 \pm 1$ \\
\cmidrule(l){2-8}
& \multirow{3}{*}{\texttt{sti}}
& \texttt{medium}  & $7 \pm 0$  & $6 \pm 1$  & $15 \pm 2$ & $10 \pm 1$ & $4 \pm 1$ \\
& & \texttt{large} & $10 \pm 1$ & $9 \pm 1$  & $18 \pm 2$ & $13 \pm 1$ & $6 \pm 0$ \\
& & \texttt{giant} & $12 \pm 1$ & $13 \pm 1$ & $27 \pm 2$ & $13 \pm 1$ & $7 \pm 1$ \\
\bottomrule
\end{tabular}

}
\vspace{-3pt}
\captionsetup{font=footnotesize}
\caption{\textbf{Collision rate (\%)} ($\downarrow$): percentage of time steps with agent-boundary contact in \texttt{MuJoCo} (mean $\pm$ std over 10 seeds). Full results appear in Appendix~\ref{app:add}.}
\label{tab:collision-small}
\vspace{-10pt}

\end{wraptable}

\textbf{Benchmarks and baselines.}
We evaluate on \texttt{OGBench}'s \texttt{pointmaze}, \texttt{antmaze}, and \texttt{humanoidmaze} environments across datasets and sizes \citep{park2025ogbench}. Maze sizes are \texttt{medium}, \texttt{large}, and \texttt{giant} and the three robots are a point, an ant, and a human. Training dataset types are \texttt{navigate}, \texttt{stitch}, and \texttt{explore}. An agent must navigate from its initial position to a specified goal. We call our method \textbf{\ours} (Reachable QRL). We primarily compare with QRL \citep{wang2023optimal} and EikHiQRL \citep{giammarino2026goal} to assess our quasimetric critic contribution. We compare with GCIQL \citep{kostrikov2021offline}, CRL \citep{eysenbach2022contrastive}, HIQL \citep{park2023hiql}, SAW \citep{zhou2026flattening}, OTA \citep{ahn2026option}, and GAS \citep{baek2025graph}. Our reproduced baseline results largely match \texttt{OGBench} and respective papers. See Appendix~\ref{app:exp}.

\textbf{Training and evaluation.}
\ours{} and EikHiQRL share the IQE critic, AWR actor, and low-level TD critic architectures, with common hyperparameters matched except for the method-specific differences mentioned in Appendix~\ref{app:exp} (along training budgets, hyperparameters, evaluation counts, and metric definitions). We report success and wall-collision rates on $10$ seeds with their mean and std.

\textbf{Comparison to quasimetric baselines.} In Table~\ref{tab:main}, \ours{} obtains higher success than QRL and EikHiQRL in almost all settings. For instance, on \texttt{humanoid-large-stitch}, \ours{} reaches $89\%$ while QRL and EikHiQRL reach $3\%$ and $63\%$ respectively. On \texttt{antmaze-giant}, \ours{} improves on EikHiQRL by $21$ points on \texttt{navigate} and $26$ points on \texttt{stitch}, and by $74$ points on \texttt{antmaze-large-explore}. Critic constraints in \ours{} are computed from the identified reachability (i.e., control dynamics and boundaries) of the data without requiring complete successful trajectories, which is consistent with its robustness in \texttt{explore}.

\textbf{Comparison to other GCRL methods.} The competitive non-quasimetric methods are GAS and OTA, both recent state of the art on \texttt{OGBench}. Overall, \ours{} is comparable to or better than both. \ours{} is largely higher on \texttt{pointmaze} datasets and \texttt{antmaze-giant-navigate}, while GAS is higher on \texttt{antmaze-large-explore} and OTA on \texttt{humanoid-giant-navigate}. \ours{} thus closes the gap between quasimetric methods and the state of the art in offline GCRL. Figure~\ref{fig:training-vs-success} further shows \ours{} on the Pareto frontier of training time and success.

\textbf{Convergence and collision rate.}
Figure~\ref{fig:humanoid-eval-curve} shows training curves on \texttt{humanoidmaze}, where \ours{} learns quickly and reaches the best or near-best final success, consistent with the explicit reachability information supplied to the critic. We also study boundary interactions during policy execution through Table~\ref{tab:collision-small}. \ours{} achieves lower mean collision rates than QRL and EikHiQRL, and lower or comparable collision rates compared with OTA and GAS across the reported \texttt{antmaze} and \texttt{humanoidmaze} settings. Full results, including mixed \texttt{pointmaze}, are in Appendix~\ref{app:add}.

\begin{table}[t]
\centering
\begin{minipage}[c]{0.35\linewidth}
  \centering
  \setlength{\fboxsep}{0pt}\setlength{\fboxrule}{0.6pt}%
    \fbox{\includegraphics[width=\dimexpr0.85\linewidth-2\fboxrule\relax]{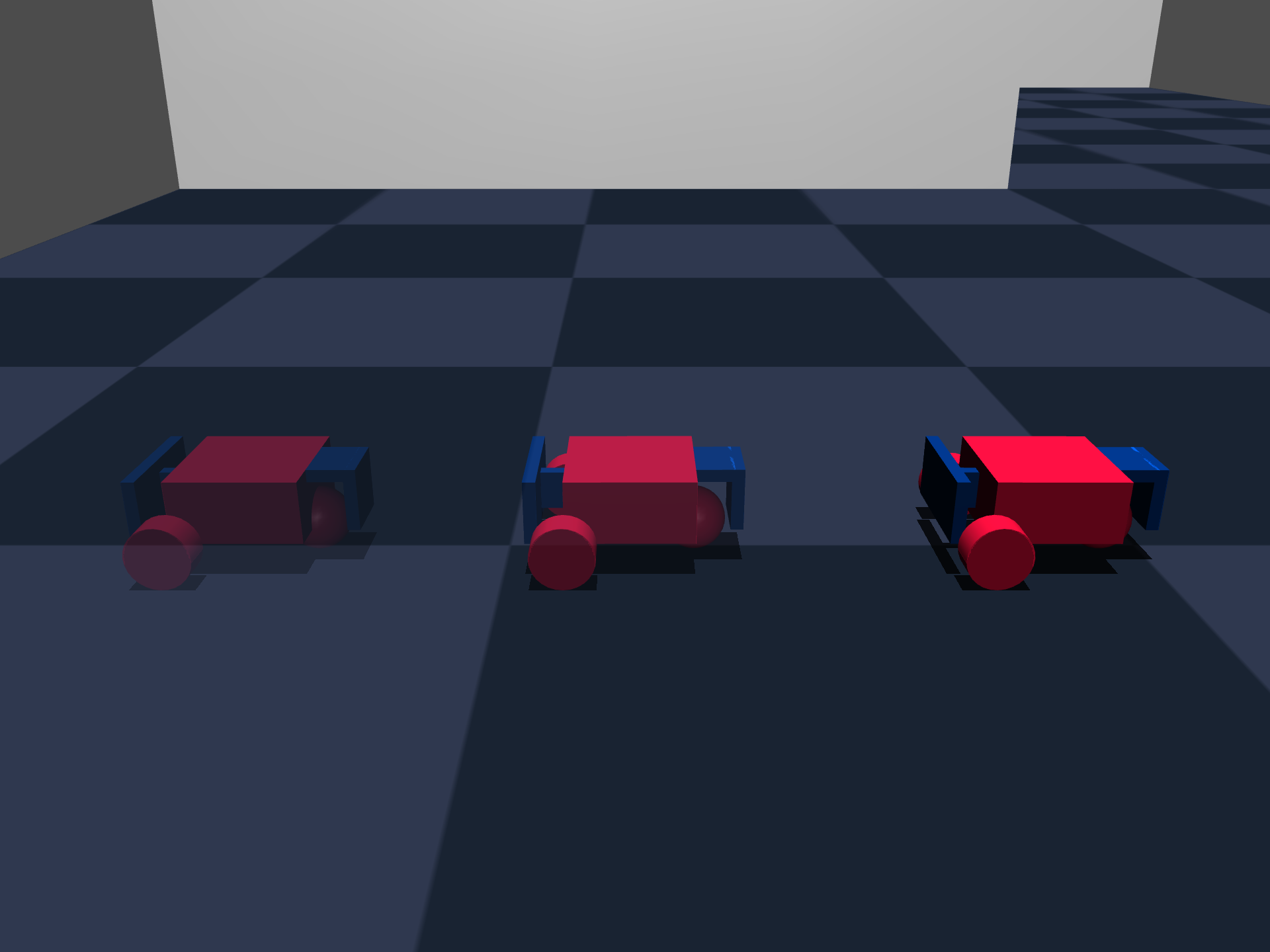}}
  \captionof{figure}{\texttt{Safety-Gymnasium} \texttt{Car} moving in \texttt{OGBench}'s maze.}
  \label{fig:carmaze}
\end{minipage}\hfill
\begin{minipage}[c]{0.6\linewidth}
  \centering
  \footnotesize
  \setlength{\tabcolsep}{4pt}
  \renewcommand{\arraystretch}{1.1}
   \scalebox{0.9}{\begin{tabular}{ccl ll @{\hspace{10pt}} ll}
  \toprule
  \multirow{2}{*}{\textbf{Env}} & \multirow{2}{*}{\textbf{Data}} & \multirow{2}{*}{\textbf{Size}}
  & \multicolumn{2}{c}{\textbf{Other GCRL}} & \multicolumn{2}{c}{\textbf{Quasimetric}} \\
  \cmidrule(lr){4-5}\cmidrule(lr){6-7}
  & & & \multicolumn{1}{c}{OTA} & \multicolumn{1}{c}{GAS}
  & \multicolumn{1}{c}{EikHiQRL} & \multicolumn{1}{c}{\textbf{\ours}} \\
  \midrule
  \multirow[c]{6}{*}{\rotatebox[origin=c]{90}{\texttt{carmaze}}}
  & \multirow[c]{3}{*}{\rotatebox[origin=c]{90}{\texttt{nav}}}
  & \texttt{medium}  & $63 \pm 7$  & $81 \pm 5$  & $72 \pm 8$  & \best{$90 \pm 3$} \\
  & & \texttt{large} & $32 \pm 14$ & $67 \pm 7$  & $45 \pm 15$ & \best{$81 \pm 12$} \\
  & & \texttt{giant} & $3 \pm 3$   & $40 \pm 8$  & $13 \pm 9$  & \best{$60 \pm 14$} \\
  \cmidrule(lr){2-7}
  & \multirow[c]{3}{*}{\rotatebox[origin=c]{90}{\texttt{sti}}}
  & \texttt{medium}  & $60 \pm 6$  & $74 \pm 6$  & $69 \pm 9$  & \best{$82 \pm 6$} \\
  & & \texttt{large} & $14 \pm 6$  & $64 \pm 10$ & $26 \pm 9$  & \best{$78 \pm 12$} \\
  & & \texttt{giant} & $0 \pm 0$   & $31 \pm 8$  & $5 \pm 7$   & \best{$45 \pm 17$} \\
  \bottomrule
  \end{tabular}}
  \vspace{-3pt}
  \caption{\textbf{Success rates (\%) ($\uparrow$)} on \texttt{carmaze} (mean $\pm$ std over 10 seeds). Full results appear in Appendix ~\ref{app:add}.}
  \label{tab:carmaze-small}
\end{minipage}
\vspace{-15pt}
\end{table}

\subsection{Nonholonomic Control Dynamics}

We additionally test our finite-horizon value supervision under nominally \emph{nonholonomic} dynamics, where orientation constrains translation (i.e., turning rather than stepping sideways). This provides a targeted setting for evaluating our approach when reachability depends on composing turning and translation. To this end, we construct the \texttt{carmaze} environment (Figure~\ref{fig:carmaze}) by integrating the \texttt{Safety-Gymnasium Car} robot \citep{ji2023safety}, natively simulated in \texttt{MuJoCo} \citep{todorov2012mujoco}, into the same \texttt{OGBench} layouts. Following its collection pipeline, an SAC locomotion expert with waypoint guidance generates \texttt{navigate} and \texttt{stitch} datasets. All methods share the datasets and evaluation protocol; car-specific adaptations and tuning details are provided in Appendix~\ref{app:add}.

\textbf{Success rate in \texttt{carmaze}.} In Table~\ref{tab:carmaze-small}, \ours{} outperforms baselines in all settings. On giant \texttt{navigate} and \texttt{stitch}, it reaches $60\%$ and $45\%$, versus $40\%$ and $31\%$ for GAS. These gains are consistent with our motivation, where finite-horizon supervision explicitly incorporates composed motions into local critic constraints when progress requires composing turning and translation. The \texttt{stitch} gains support its usefulness for reaching distant goals from short trajectory segments.
\section{Related Work}\label{sec:related}

\textbf{Goal-conditioned RL.} GCRL provides a framework for learning reusable policies across goals \citep{schaul2015universal, pong2018temporal} and turning achieved outcomes into self-supervision through hindsight relabeling \citep{andrychowicz2017hindsight, ghosh2019learning}. Offline GCRL extends this to reward-free datasets without further interaction \citep{chebotar2021actionable, park2025ogbench}. Approaches include weighted imitation \citep{yang2022rethinking}, policy regression via occupancy matching \citep{ma2022far}, and contrastive value learning \citep{eysenbach2022contrastive}. To tackle longer horizons, others use subgoal reaching \citep{park2023hiql}, advantage estimates via temporal abstraction \citep{ahn2026option}, flat-policy bootstrapping  \citep{zhou2026flattening}, and graph-based stitching \citep{baek2025graph}. A complementary line exploits the metric and geometric structure of goal-reaching values.

\textbf{(Quasi)metrics in GCRL.} Viewing costs as distances motivates reward shaping via dynamical distances \citep{hartikainen2019dynamical} and subgoal search \citep{eysenbach2019search}, planning embeddings for navigation \citep{yang2020plan2vec}, and metric representations for skill discovery \citep{park2024foundation}. A key insight is that \emph{quasimetrics} encode both directional costs and path composition. QRL turns this into a learning principle linking local transition constraints to global distances \citep{wang2023optimal}. CMD and TMD combine quasimetric structure with contrastive temporal-distance learning \citep{myers2024learning, myers2026offline}. EikHiQRL regularizes critic gradients with Eikonal constraints under simplified dynamics \citep{giammarino2026goal}. Grounding local supervision in composed reachability and state feasibility remains challenging, motivating geometric and optimal control.

\textbf{Geometric and optimal control.} Geometric control shows composing controls can generate motion beyond the instantaneous span \citep{jean2014control}. In optimal control, Hamilton-Jacobi theory connects dynamics to value gradients \citep{bardi1997optimal}, while state-constraint boundary conditions encode feasibility \citep{soner1986optimal}. In RL, these ideas motivate HJB-based derivative regularization \citep{lien2024enhancing} and Eikonal-based TD learning \citep{giammarino2026physics}. \ours{} unifies these in local supervision: a finite-horizon Hamiltonian preserves control composition while linearizing only the critic and boundary constraints encode feasibility.

\section{Conclusion}\label{sec:conc}
We formulated finite-horizon Hamiltonian constraints for learning global goal-reaching geometry through local quasimetric critic supervision. Our approach retains the effects of controls composed over time through the reachable set. We further separate dynamical reachability from boundary geometry, providing complementary signals about control dynamics and feasibility. With this, we developed \ours{} for GCRL, with data-driven zonotopic reachability estimates and boundary regularization. Experiments on \texttt{OGBench} show improvements over quasimetric baselines and competitive performance against other GCRL methods. We discuss limitations in Appendix~\ref{app:limit}.
\section*{Acknowledgments}
Authors were all partly supported by NIH R01AG092220. 

\section*{AI Use Disclosure}
In this work, we used generative AI tools to aid in refining writing, to aid in the generation of figures, to aid in literature search, to aid in editing the presentation of the proofs we derived, and to aid in coding with LLM agents. We have not used generative AI tools for proving mathematical claims or writing beyond the paragraph scale. We have reviewed all AI-assisted work and have verified the correctness of writing and code. We take responsibility for the final content of this work, including text, claims or artifacts produced with the aid of generative AI.

\bibliography{reference}

\begin{thebibliography}{49}
\providecommand{\natexlab}[1]{#1}
\providecommand{\url}[1]{\texttt{#1}}
\expandafter\ifx\csname urlstyle\endcsname\relax
  \providecommand{\doi}[1]{doi: #1}\else
  \providecommand{\doi}{doi: \begingroup \urlstyle{rm}\Url}\fi

\bibitem[Agrachev et~al.(2020)Agrachev, Barilari, and Boscain]{agrachev2020comprehensive}
Andrei Agrachev, Davide Barilari, and Ugo Boscain.
\newblock \emph{A comprehensive introduction to sub-Riemannian geometry}.
\newblock Cambridge University Press, 2020.

\bibitem[Ahn et~al.(2026)Ahn, Choi, Han, and Moon]{ahn2026option}
Hongjoon Ahn, Heewoong Choi, Jisu Han, and Taesup Moon.
\newblock Option-aware temporally abstracted value for offline goal-conditioned reinforcement learning.
\newblock \emph{Advances in Neural Information Processing Systems}, 38:\penalty0 99833--99861, 2026.

\bibitem[Althoff \& Frehse(2016)Althoff and Frehse]{althoff2016combining}
Matthias Althoff and Goran Frehse.
\newblock Combining zonotopes and support functions for efficient reachability analysis of linear systems.
\newblock In \emph{2016 IEEE 55th Conference on Decision and Control (CDC)}, pp.\  7439--7446. IEEE, 2016.

\bibitem[Althoff et~al.(2021)Althoff, Frehse, and Girard]{althoff2021set}
Matthias Althoff, Goran Frehse, and Antoine Girard.
\newblock Set propagation techniques for reachability analysis.
\newblock \emph{Annual Review of Control, Robotics, and Autonomous Systems}, 4\penalty0 (1):\penalty0 369--395, 2021.

\bibitem[Andrychowicz et~al.(2017)Andrychowicz, Wolski, Ray, Schneider, Fong, Welinder, McGrew, Tobin, Pieter~Abbeel, and Zaremba]{andrychowicz2017hindsight}
Marcin Andrychowicz, Filip Wolski, Alex Ray, Jonas Schneider, Rachel Fong, Peter Welinder, Bob McGrew, Josh Tobin, OpenAI Pieter~Abbeel, and Wojciech Zaremba.
\newblock Hindsight experience replay.
\newblock \emph{Advances in neural information processing systems}, 30, 2017.

\bibitem[Aubin et~al.(2011)Aubin, Bayen, and Saint-Pierre]{aubin2011viability}
Jean-Pierre Aubin, Alexandre~M. Bayen, and Patrick Saint-Pierre.
\newblock \emph{Viability Theory: New Directions}.
\newblock Springer, 2 edition, 2011.
\newblock \doi{10.1007/978-3-642-16684-6}.

\bibitem[Baek et~al.(2025)Baek, Park, Park, Oh, and Kim]{baek2025graph}
Seungho Baek, Taegeon Park, Jongchan Park, Seungjun Oh, and Yusung Kim.
\newblock Graph-assisted stitching for offline hierarchical reinforcement learning.
\newblock \emph{arXiv preprint arXiv:2506.07744}, 2025.

\bibitem[Bardi et~al.(1997)Bardi, Dolcetta, et~al.]{bardi1997optimal}
Martino Bardi, Italo~Capuzzo Dolcetta, et~al.
\newblock \emph{Optimal control and viscosity solutions of Hamilton-Jacobi-Bellman equations}, volume~12.
\newblock Springer, 1997.

\bibitem[Blanchini et~al.(2008)Blanchini, Miani, et~al.]{blanchini2008set}
Franco Blanchini, Stefano Miani, et~al.
\newblock \emph{Set-theoretic methods in control}, volume~78.
\newblock Springer, 2008.

\bibitem[Chebotar et~al.(2021)Chebotar, Hausman, Lu, Xiao, Kalashnikov, Varley, Irpan, Eysenbach, Julian, Finn, et~al.]{chebotar2021actionable}
Yevgen Chebotar, Karol Hausman, Yao Lu, Ted Xiao, Dmitry Kalashnikov, Jake Varley, Alex Irpan, Benjamin Eysenbach, Ryan Julian, Chelsea Finn, et~al.
\newblock Actionable models: Unsupervised offline reinforcement learning of robotic skills.
\newblock \emph{arXiv preprint arXiv:2104.07749}, 2021.

\bibitem[Chen(1984)]{chen1984linear}
Chi-Tsong Chen.
\newblock \emph{Linear system theory and design}, volume 301.
\newblock Holt, Rinehart and Winston New York, 1984.

\bibitem[Delfour \& Zol{\'e}sio(2011)Delfour and Zol{\'e}sio]{delfour2011shapes}
Michel~C Delfour and J-P Zol{\'e}sio.
\newblock \emph{Shapes and geometries: metrics, analysis, differential calculus, and optimization}.
\newblock SIAM, 2011.

\bibitem[Eysenbach et~al.(2019)Eysenbach, Salakhutdinov, and Levine]{eysenbach2019search}
Ben Eysenbach, Russ~R Salakhutdinov, and Sergey Levine.
\newblock Search on the replay buffer: Bridging planning and reinforcement learning.
\newblock \emph{Advances in neural information processing systems}, 32, 2019.

\bibitem[Eysenbach et~al.(2022)Eysenbach, Zhang, Levine, and Salakhutdinov]{eysenbach2022contrastive}
Benjamin Eysenbach, Tianjun Zhang, Sergey Levine, and Russ~R Salakhutdinov.
\newblock Contrastive learning as goal-conditioned reinforcement learning.
\newblock \emph{Advances in Neural Information Processing Systems}, 35:\penalty0 35603--35620, 2022.

\bibitem[Ghosh et~al.(2019)Ghosh, Gupta, Reddy, Fu, Devin, Eysenbach, and Levine]{ghosh2019learning}
Dibya Ghosh, Abhishek Gupta, Ashwin Reddy, Justin Fu, Coline Devin, Benjamin Eysenbach, and Sergey Levine.
\newblock Learning to reach goals via iterated supervised learning.
\newblock \emph{arXiv preprint arXiv:1912.06088}, 2019.

\bibitem[Giammarino \& Qureshi(2026)Giammarino and Qureshi]{giammarino2026goal}
Vittorio Giammarino and Ahmed~Hussain Qureshi.
\newblock Goal reaching with eikonal-constrained hierarchical quasimetric reinforcement learning.
\newblock In \emph{International Conference on Learning Representations}, volume 2026, pp.\  59145--59182, 2026.

\bibitem[Giammarino et~al.(2026)Giammarino, Ni, and Qureshi]{giammarino2026physics}
Vittorio Giammarino, Ruiqi Ni, and Ahmed Qureshi.
\newblock Physics-informed value learner for offline goal-conditioned reinforcement learning.
\newblock \emph{Advances in Neural Information Processing Systems}, 38:\penalty0 137769--137796, 2026.

\bibitem[Girard(2005)]{girard2005reachability}
Antoine Girard.
\newblock Reachability of uncertain linear systems using zonotopes.
\newblock In \emph{HSCC}, volume 3414, pp.\  291--305, 2005.

\bibitem[Haralick et~al.(1987)Haralick, Sternberg, and Zhuang]{haralick1987image}
Robert~M Haralick, Stanley~R Sternberg, and Xinhua Zhuang.
\newblock Image analysis using mathematical morphology.
\newblock \emph{IEEE transactions on pattern analysis and machine intelligence}, \penalty0 (4):\penalty0 532--550, 1987.

\bibitem[Hartikainen et~al.(2019)Hartikainen, Geng, Haarnoja, and Levine]{hartikainen2019dynamical}
Kristian Hartikainen, Xinyang Geng, Tuomas Haarnoja, and Sergey Levine.
\newblock Dynamical distance learning for semi-supervised and unsupervised skill discovery.
\newblock \emph{arXiv preprint arXiv:1907.08225}, 2019.

\bibitem[Jean(2014)]{jean2014control}
Fr{\'e}d{\'e}ric Jean.
\newblock \emph{Control of nonholonomic systems: from sub-Riemannian geometry to motion planning}.
\newblock Springer, 2014.

\bibitem[Ji et~al.(2023)Ji, Zhang, Zhou, Pan, Huang, Sun, Geng, Zhong, Dai, and Yang]{ji2023safety}
Jiaming Ji, Borong Zhang, Jiayi Zhou, Xuehai Pan, Weidong Huang, Ruiyang Sun, Yiran Geng, Yifan Zhong, Josef Dai, and Yaodong Yang.
\newblock Safety gymnasium: A unified safe reinforcement learning benchmark.
\newblock \emph{Advances in Neural Information Processing Systems}, 36:\penalty0 18964--18993, 2023.

\bibitem[Kostrikov et~al.(2021)Kostrikov, Nair, and Levine]{kostrikov2021offline}
Ilya Kostrikov, Ashvin Nair, and Sergey Levine.
\newblock Offline reinforcement learning with implicit q-learning.
\newblock \emph{arXiv preprint arXiv:2110.06169}, 2021.

\bibitem[Liberzon(2011)]{liberzon2011calculus}
Daniel Liberzon.
\newblock Calculus of variations and optimal control theory: a concise introduction.
\newblock 2011.

\bibitem[Lien et~al.(2024)Lien, Hsieh, Li, and Wang]{lien2024enhancing}
Yun-Hsuan Lien, Ping-Chun Hsieh, Tzu-Mao Li, and Yu-Shuen Wang.
\newblock Enhancing value function estimation through first-order state-action dynamics in offline reinforcement learning.
\newblock In \emph{Forty-first International Conference on Machine Learning}, 2024.

\bibitem[Lillicrap et~al.(2016)Lillicrap, Hunt, Pritzel, Heess, Erez, Tassa, Silver, and Wierstra]{lillicrap2015continuous}
Timothy~P Lillicrap, Jonathan~J Hunt, Alexander Pritzel, Nicolas Heess, Tom Erez, Yuval Tassa, David Silver, and Daan Wierstra.
\newblock Continuous control with deep reinforcement learning.
\newblock In \emph{International Conference on Learning Representations}, 2016.
\newblock URL \url{https://arxiv.org/abs/1509.02971}.

\bibitem[Ljung(1999)]{ljung1999system}
L.~Ljung.
\newblock \emph{System Identification: Theory for the User}.
\newblock Prentice Hall information and system sciences series. Prentice Hall PTR, 1999.
\newblock ISBN 9780136566953.
\newblock URL \url{https://books.google.com/books?id=nHFoQgAACAAJ}.

\bibitem[Ma et~al.(2022)Ma, Yan, Jayaraman, and Bastani]{ma2022far}
Yecheng~Jason Ma, Jason Yan, Dinesh Jayaraman, and Osbert Bastani.
\newblock How far i'll go: Offline goal-conditioned reinforcement learning via $ f $-advantage regression.
\newblock \emph{arXiv preprint arXiv:2206.03023}, 2022.

\bibitem[Myers et~al.(2024)Myers, Zheng, Dragan, Levine, and Eysenbach]{myers2024learning}
Vivek Myers, Chongyi Zheng, Anca Dragan, Sergey Levine, and Benjamin Eysenbach.
\newblock Learning temporal distances: Contrastive successor features can provide a metric structure for decision-making.
\newblock \emph{arXiv preprint arXiv:2406.17098}, 2024.

\bibitem[Myers et~al.(2026)Myers, Zheng, Eysenbach, and Levine]{myers2026offline}
Vivek Myers, Bill Zheng, Benjamin Eysenbach, and Sergey Levine.
\newblock Offline goal-conditioned reinforcement learning with quasimetric representations.
\newblock \emph{Advances in Neural Information Processing Systems}, 38:\penalty0 19654--19679, 2026.

\bibitem[Palis(1974)]{palis1974vector}
Jacob Palis.
\newblock Vector fields generate few diffeomorphisms.
\newblock \emph{Bulletin of the American Mathematical Society}, 80\penalty0 (3):\penalty0 503--505, 1974.

\bibitem[Park et~al.(2023)Park, Ghosh, Eysenbach, and Levine]{park2023hiql}
Seohong Park, Dibya Ghosh, Benjamin Eysenbach, and Sergey Levine.
\newblock Hiql: Offline goal-conditioned rl with latent states as actions.
\newblock \emph{Advances in Neural Information Processing Systems}, 36:\penalty0 34866--34891, 2023.

\bibitem[Park et~al.(2024)Park, Kreiman, and Levine]{park2024foundation}
Seohong Park, Tobias Kreiman, and Sergey Levine.
\newblock Foundation policies with hilbert representations.
\newblock \emph{arXiv preprint arXiv:2402.15567}, 2024.

\bibitem[Park et~al.(2025)Park, Frans, Eysenbach, and Levine]{park2025ogbench}
Seohong Park, Kevin Frans, Benjamin Eysenbach, and Sergey Levine.
\newblock Ogbench: Benchmarking offline goal-conditioned rl.
\newblock In \emph{International Conference on Learning Representations}, volume 2025, pp.\  94937--94982, 2025.

\bibitem[Peng et~al.(2019)Peng, Kumar, Zhang, and Levine]{peng2019advantage}
Xue~Bin Peng, Aviral Kumar, Grace Zhang, and Sergey Levine.
\newblock Advantage-weighted regression: Simple and scalable off-policy reinforcement learning.
\newblock \emph{arXiv preprint arXiv:1910.00177}, 2019.

\bibitem[Pong et~al.(2018)Pong, Gu, Dalal, and Levine]{pong2018temporal}
Vitchyr Pong, Shixiang Gu, Murtaza Dalal, and Sergey Levine.
\newblock Temporal difference models: Model-free deep rl for model-based control.
\newblock \emph{arXiv preprint arXiv:1802.09081}, 2018.

\bibitem[Rawlings et~al.(2020)Rawlings, Mayne, and Diehl]{rawlings2020model}
James~B Rawlings, David~Q Mayne, and Moritz~M Diehl.
\newblock Model predictive control: theory, computation, and design.
\newblock \emph{(No Title)}, 2020.

\bibitem[Schaul et~al.(2015)Schaul, Horgan, Gregor, and Silver]{schaul2015universal}
Tom Schaul, Daniel Horgan, Karol Gregor, and David Silver.
\newblock Universal value function approximators.
\newblock In \emph{International conference on machine learning}, pp.\  1312--1320. PMLR, 2015.

\bibitem[Soner(1986)]{soner1986optimal}
Halil~Mete Soner.
\newblock Optimal control with state-space constraint i.
\newblock \emph{SIAM Journal on Control and Optimization}, 24\penalty0 (3):\penalty0 552--561, 1986.

\bibitem[Sutton et~al.(1999)Sutton, Precup, and Singh]{sutton1999between}
Richard~S Sutton, Doina Precup, and Satinder Singh.
\newblock Between mdps and semi-mdps: A framework for temporal abstraction in reinforcement learning.
\newblock \emph{Artificial intelligence}, 112\penalty0 (1-2):\penalty0 181--211, 1999.

\bibitem[Tebjou et~al.(2023)Tebjou, Frehse, et~al.]{tebjou2023data}
Abdelmouaiz Tebjou, Goran Frehse, et~al.
\newblock Data-driven reachability using christoffel functions and conformal prediction.
\newblock In \emph{Conformal and Probabilistic Prediction with Applications}, pp.\  194--213. PMLR, 2023.

\bibitem[Todorov et~al.(2012)Todorov, Erez, and Tassa]{todorov2012mujoco}
Emanuel Todorov, Tom Erez, and Yuval Tassa.
\newblock Mujoco: A physics engine for model-based control.
\newblock In \emph{2012 IEEE/RSJ international conference on intelligent robots and systems}, pp.\  5026--5033. IEEE, 2012.

\bibitem[Verhaegen \& Verdult(2007)Verhaegen and Verdult]{verhaegen2007filtering}
Michel Verhaegen and Vincent Verdult.
\newblock \emph{Filtering and system identification: a least squares approach}.
\newblock Cambridge university press, 2007.

\bibitem[Wang \& Isola(2022)Wang and Isola]{wang2022improved}
Tongzhou Wang and Phillip Isola.
\newblock Improved representation of asymmetrical distances with interval quasimetric embeddings.
\newblock In \emph{NeurIPS 2022 Workshop on Symmetry and Geometry in Neural Representations}, 2022.
\newblock URL \url{https://openreview.net/forum?id=KRiST_rzkGl}.

\bibitem[Wang et~al.(2023)Wang, Torralba, Isola, and Zhang]{wang2023optimal}
Tongzhou Wang, Antonio Torralba, Phillip Isola, and Amy Zhang.
\newblock Optimal goal-reaching reinforcement learning via quasimetric learning.
\newblock In \emph{International Conference on Machine Learning}, pp.\  36411--36430. PMLR, 2023.

\bibitem[Yang et~al.(2020)Yang, Zhang, Morcos, Pineau, Abbeel, and Calandra]{yang2020plan2vec}
Ge~Yang, Amy Zhang, Ari Morcos, Joelle Pineau, Pieter Abbeel, and Roberto Calandra.
\newblock Plan2vec: Unsupervised representation learning by latent plans.
\newblock In \emph{Learning for Dynamics and Control}, pp.\  935--946. PMLR, 2020.

\bibitem[Yang et~al.(2022)Yang, Lu, Li, Sun, Fang, Du, Li, Han, and Zhang]{yang2022rethinking}
Rui Yang, Yiming Lu, Wenzhe Li, Hao Sun, Meng Fang, Yali Du, Xiu Li, Lei Han, and Chongjie Zhang.
\newblock Rethinking goal-conditioned supervised learning and its connection to offline rl.
\newblock \emph{arXiv preprint arXiv:2202.04478}, 2022.

\bibitem[Yue et~al.(2016)Yue, Thunberg, and Gon{\c{c}}alves]{yue2016inverse}
Zuogon Yue, Johan Thunberg, and Jorge Gon{\c{c}}alves.
\newblock Inverse problems for matrix exponential in system identification: system aliasing.
\newblock \emph{arXiv preprint arXiv:1605.06973}, 2016.

\bibitem[Zhou \& Kao(2026)Zhou and Kao]{zhou2026flattening}
John Zhou and Jonathan Kao.
\newblock Flattening hierarchies with policy bootstrapping.
\newblock \emph{Advances in Neural Information Processing Systems}, 38:\penalty0 82975--83007, 2026.

\end{thebibliography}
\bibliographystyle{arxiv_ref}

\newpage
\appendix
\section*{Appendix}

\section{Limitations and Future Work}\label{app:limit}
We see several natural directions for extending our method. Our instantiation in Section~\ref{sec:offline-inst} uses an affine dynamics model and occupancy-based boundary estimates to realize tractable evaluation of the constraints in (\ref{eq:normalized-objective}). While these approximations are effective in our \texttt{OGBench} experiments, the affine model may not capture all state-dependent and nonlinear effects of the underlying control dynamics. The estimates also depend on coverage of relevant states in the training data $\mathcal D$, a concern shared by offline baselines such as QRL and EikHiQRL \citep{wang2023optimal,giammarino2026goal}. Richer state representations and nonlinear reachability estimators could broaden applicability; an important challenge is retaining efficient constraint evaluation. Our formulation in Section~\ref{sec:method} also assumes deterministic dynamics, and the experiments in Section~\ref{sec:exp} evaluate offline tasks. Extending the formulation to stochastic control systems and evaluating the method with online data collection are promising directions for future work.
\section{Theoretical Analysis}\label{app:proof}

We first show that expressing Bellman optimality through controls composed over duration $\tau$ preserves the target goal-reaching value $d^\ast$ and its quasimetric structure. This provides the foundation for our finite-horizon critic constraints.

\begin{lemma}[Goal-reaching geometry under compositional control]
\label{thm:k-step-quasimetric}
Consider the minimum-cost $d^\ast$, $\mathcal U_\tau(s)$, and $\mathcal T_\tau$ from \eqref{eq:tau-bellman}, and $\mathcal R_\tau(s)$ from \eqref{eq:true-reachability}. Assume $d^\ast$ is finite and feasible controls can be joined and split at supported times. Then $d^\ast\in\mathcal Q(\mathcal S)$ and satisfies \eqref{eq:tau-bellman-reachability} whenever $d^\ast(s,g)>\tau$. Moreover, $d^\ast$ is the pointwise greatest quasimetric satisfying $d(s,\mathcal T_\tau(s,u))\leq \tau$ for every supported $\tau>0$, $s\in\mathcal S$, and $u\in\mathcal U_\tau(s)$.
\end{lemma}
\begin{proof}
Nonnegativity and $d^\ast(s,s)=0$ follow from the definition. Joining arbitrarily near-optimal controls through an intermediate state gives
\[
    d^\ast(s_1,s_3)\leq d^\ast(s_1,s_2)+d^\ast(s_2,s_3).
\]
Since $d^\ast$ is finite, $d^\ast\in\mathcal Q(\mathcal S)$. Every $u\in\mathcal U_\tau(s)$ reaches $\mathcal T_\tau(s,u)$ in duration $\tau$, so $d^\ast(s,\mathcal T_\tau(s,u))\leq\tau$. Conversely, any quasimetric $d$ satisfying these bounds obeys $d(s,g)\leq\tau$ whenever $\mathcal T_\tau(s,u)=g$. Taking the infimum over feasible goal-reaching controls gives $d(s,g)\leq d^\ast(s,g)$. For Bellman optimality, joining a control in $\mathcal U_\tau(s)$ with a near-optimal continuation gives
\[
    d^\ast(s,g)\leq\tau+
    \inf_{u\in\mathcal U_\tau(s)}
    d^\ast(\mathcal T_\tau(s,u),g).
\]
Conversely, when $d^\ast(s,g)>\tau$, every near-optimal goal-reaching trajectory extends beyond $\tau$. Splitting it there, its remaining duration is at least the optimal cost from $\mathcal T_\tau(s,u)$. Taking arbitrarily near-optimal trajectories gives the reverse inequality. Substituting \eqref{eq:true-reachability} yields \eqref{eq:tau-bellman-reachability}.
\end{proof}

Lemma~\ref{thm:k-step-quasimetric} establishes that finite-horizon reachability describes the same optimal goal-reaching value. We next connect this formulation to the transition data in Section~\ref{sec:offline-inst} and show that the normalization used in \eqref{eq:normalized-objective} changes only the cost units.

\begin{lemma}[Data instantiation and cost normalization]
\label{cor:cost-normalization}
Under the instantiation $\tau=k\Delta t$ in Section~\ref{sec:offline-inst}, with goal reaching evaluated at each environment step, dividing time costs by $\tau$ preserves optimal policies and expresses costs in environment steps divided by $k$. Scaling $d_\theta\mapsto d_\theta/\tau$ converts \eqref{eq:our-objective} to \eqref{eq:normalized-objective}, preserving quasimetric structure and the boundary condition.
\end{lemma}

\begin{proof}
Under the stated assumption that each action takes duration $\Delta t$, an action sequence of length $k$ takes duration $\tau$ and reaches the endpoint obtained by composing $\mathcal T$ $k$ times. The reachable-set identities follow. For any policy, a trajectory reaching the goal after $N$ environment steps incurs time cost $N\Delta t$. Its normalized cost is
\[
    \frac{N\Delta t}{\tau}=\frac{N}{k}.
\]
This positive scaling preserves policy ordering, including after taking expectations. It therefore preserves the optimal policies. Since $\tau$ is fixed,
\[
    H_\tau^{\mathrm{dyn}}(s,\nabla_s d_\theta(s,g))
    =
    \max_{\delta\in\mathcal R_k^{\mathrm{dyn}}(s)}
    \left\langle-\nabla_s\!\left(\frac{d_\theta(s,g)}{\tau}\right),\delta\right\rangle-1.
\]
Positive scaling also preserves the quasimetric axioms, the boundary condition, and the ordering of the distance-maximization objective. Reusing $d_\theta$ for the normalized critic, as in Section~\ref{sec:offline-inst}, gives \eqref{eq:normalized-objective}. The numerical value of $\Delta t$ cancels.
\end{proof}

As discussed in Section~\ref{sec:sec4}, we now quantify the Hamiltonian gap introduced by replacing feasible reachability with dynamical reachability on the interior with a separate boundary condition.

\begin{lemma}[Dynamics-feasibility gap]
\label{lem:feasibility-gap}
Fix $s\in\mathcal S$ and $\tau>0$. Assume the sets in \eqref{eq:true-reachability} and \eqref{eq:rk-dyn} are nonempty and compact, with $\mathcal R_\tau(s)\subseteq\mathcal R_\tau^{\mathrm{dyn}}(s)$ as in our formulation. Define
\[
    e_\tau(s):=
    \max_{\delta\in\mathcal R_\tau^{\mathrm{dyn}}(s)}
    \min_{\bar\delta\in\mathcal R_\tau(s)}
    \|\delta-\bar\delta\|_2.
\]
Then, for every $p\in\mathbb R^n$,
\[
    0\leq H_\tau^{\mathrm{dyn}}(s,p)-H_\tau(s,p)
    \leq \frac{\|p\|_2e_\tau(s)}{\tau}.
\]
The two Hamiltonians agree if and only if some maximizer over $\mathcal R_\tau^{\mathrm{dyn}}(s)$ belongs to $\mathcal R_\tau(s)$.
\end{lemma}

\begin{proof}
Set inclusion gives the lower bound. By compactness, choose a maximizing displacement $\delta^\ast\in\mathcal R_\tau^{\mathrm{dyn}}(s)$ and a nearest $\bar\delta\in\mathcal R_\tau(s)$. Then
\[
\begin{aligned}
    \tau\bigl(H_\tau^{\mathrm{dyn}}(s,p)-H_\tau(s,p)\bigr)
    &=
    \langle-p,\delta^\ast\rangle
    -\max_{\delta\in\mathcal R_\tau(s)}\langle-p,\delta\rangle\\
    &\leq \langle-p,\delta^\ast-\bar\delta\rangle\\
    &\leq \|p\|_2e_\tau(s).
\end{aligned}
\]
If an unconstrained maximizer belongs to the feasible set, the maxima agree. Conversely, if the maxima agree, any feasible maximizer, which exists by compactness, also maximizes over the unconstrained set.
\end{proof}

Here, $e_\tau(s)$ measures the largest distance from a dynamically reachable displacement to the feasible reachable set. The Hamiltonian gap vanishes whenever an unconstrained maximizer is feasible; otherwise, the bound quantifies its dependence on reachability mismatch, critic gradient, and horizon. 

\begin{theorem}[Value approximation under true reachability]
\label{thm:global-value}
Under the assumptions of Lemma~\ref{thm:k-step-quasimetric}, fix $g\in\mathcal S$ and a supported horizon $\tau>0$. Assume each $\mathcal R_\tau(s)$ is nonempty and compact. Let $d_\theta(\cdot,g)$ be continuous, finite, and nonnegative, with $d_\theta(g,g)=0$. Let $\mathcal B_g\subseteq\mathcal S$ contain $\{s:d^\ast(s,g)\leq\tau\}$, and assume $|d_\theta(s,g)-d^\ast(s,g)|\leq\beta$ on $\mathcal B_g$. Outside $\mathcal B_g$, assume the constraint in \eqref{eq:tau-hj} holds up to residual
\[
    |H_\tau(s,\nabla_s d_\theta(s,g))|\leq\epsilon.
\]
Assume also that $\|\delta\|_2\leq\rho$ for every $s\notin\mathcal B_g$ and $\delta\in\mathcal R_\tau(s)$, and that $d_\theta(\cdot,g)$ has an $L$-Lipschitz gradient on an open set containing every segment $[s,s+\delta]$ for these pairs. If $\beta,\epsilon,L,\rho\geq0$ and $\epsilon+L\rho^2/(2\tau)<1$, then, for every $s\in\mathcal S$,
\[
    |d_\theta(s,g)-d^\ast(s,g)|
    \leq
    \left(\epsilon+\frac{L\rho^2}{2\tau}\right)d^\ast(s,g)+\beta.
\]
\end{theorem}
\begin{proof}
Fix $g\in\mathcal S$, and write $f(s)=d_\theta(s,g)$ and
$\eta=\epsilon+L\rho^2/(2\tau)<1$.
For $s\notin\mathcal B_g$ and $\delta\in\mathcal R_\tau(s)$,
the Lipschitz-gradient assumption gives
\[
    |f(s+\delta)-f(s)-\nabla f(s)^\top\delta|
    \leq \frac{L}{2}\|\delta\|_2^2
    \leq \frac{L\rho^2}{2}.
\]
Taking maxima and using the Hamiltonian residual bound yields
\[
    \tau(1-\eta)
    \leq f(s)-\min_{\delta\in\mathcal R_\tau(s)}f(s+\delta)
    \leq \tau(1+\eta).
\]
Continuity and compactness ensure that the minimum is attained.

For the upper bound, fix $s\notin\mathcal B_g$ and take any feasible goal-reaching trajectory of duration $T$. Sample it at times $0,\tau,2\tau,\ldots$ up to $T$, and let $b$ be its first sampled state in $\mathcal B_g$, reached after $N$ intervals. Such a state exists because the last sampled state has remaining goal-reaching time at most $\tau$. Since $d^\ast(b,g)\leq T-N\tau$, the upper inequality along these intervals gives
\[
    f(s)\leq(1+\eta)N\tau+f(b)
    \leq(1+\eta)N\tau+d^\ast(b,g)+\beta
    \leq(1+\eta)T+\beta.
\]
Taking the infimum over feasible durations $T$ yields
$f(s)\leq(1+\eta)d^\ast(s,g)+\beta$.

For the lower bound, start at $s_0=s$ and choose
$s_{i+1}\in\arg\min_{y\in s_i+\mathcal R_\tau(s_i)}f(y)$
until entering $\mathcal B_g$. Before entry,
$f(s_{i+1})\leq f(s_i)-\tau(1-\eta)$.
Since $f\geq0$ and $\tau(1-\eta)>0$, entry must occur after finitely many intervals. Let $b$ be the entry state after $N$ intervals. Concatenation gives $d^\ast(s,g)\leq N\tau+d^\ast(b,g)$, and hence
\[
    f(s)\geq(1-\eta)N\tau+f(b)
    \geq(1-\eta)N\tau+d^\ast(b,g)-\beta
    \geq(1-\eta)d^\ast(s,g)-\beta.
\]
Combining the bounds proves the claim outside $\mathcal B_g$. On $\mathcal B_g$, it follows directly from the assumed error bound.
\end{proof}

We note that the above theorem provides global value error bounds in terms of the constraint residual, the first-order approximation error, and the assumed accuracy on $\mathcal B_g$. While this bound gives some insight for the value recovery of our method in the idealized setting, this analysis should thus be interpreted only as conditional guarantee for the true reachable set. The practical feasibility and effectiveness of the method are more directly established via our experiments in Section~\ref{sec:exp}.
\section{Main Experiment Details}\label{app:exp}

\subsection{Benchmarks}

\textbf{Benchmarks and datasets.}
We use the state-based environments from \texttt{OGBench} \citep{park2025ogbench}, simulated in \texttt{MuJoCo} \citep{todorov2012mujoco}. Our main experiments cover \texttt{pointmaze}, \texttt{antmaze}, and \texttt{humanoidmaze} across \texttt{medium}, \texttt{large}, and \texttt{giant} layouts. 

In \texttt{pointmaze}, a point mass moves through the maze, with its 2-D position as the state and 2-D actions directly controlling planar movement. In \texttt{antmaze}, a quadrupedal ant walks toward the goal using 8 joint-torque controls and a 29-D state. In \texttt{humanoidmaze}, a humanoid must coordinate balance and walking using 21 joint-torque controls and a 69-D state. The latter two states include position, body configuration, and motion information. All action spaces are boxes, $\mathcal A=[-1,1]^m$, where $m$ is the action dimension. We use the officially released \texttt{navigate} and \texttt{stitch} training datasets for all three families, and \texttt{explore} for \texttt{antmaze-medium} and \texttt{antmaze-large}. The \texttt{navigate} datasets contain trajectories pursuing randomly sampled goals, \texttt{stitch} contains short goal-reaching trajectories requiring composition to solve longer tasks, and \texttt{explore} contains noisy exploratory trajectories generated from random directional commands. Table~\ref{tab:benchmark-details} summarizes observation and action dimensions and training dataset sizes. All methods use the same training datasets and number of training updates within each setting. Appendix~\ref{app:add} additionally reports results on \texttt{antsoccer-arena-navigate} and \texttt{antsoccer-medium-navigate}. These tasks require an ant to approach and dribble a ball to a goal location, in an open arena or a maze, respectively. Success depends on the ball's position. 

\begin{table}[h]
\centering
\small
\begin{tabular}{lrrl}
\toprule
Environment & $n$ & $m$ & Training transitions per dataset \\
\midrule
\texttt{pointmaze} & 2 & 2
& 1M (\texttt{navigate}, \texttt{stitch}) \\
\texttt{antmaze} & 29 & 8
& 1M (\texttt{navigate}, \texttt{stitch}); 5M (\texttt{explore}) \\
\texttt{humanoidmaze} & 69 & 21
& 2M (\texttt{medium}, \texttt{large}); 4M (\texttt{giant}) \\
\bottomrule
\end{tabular}
\caption{Observation/action dimensions and training dataset sizes. M denotes one million transitions; counts exclude terminal-only observations. $n$ and $m$ denote the full state and action spaces, respectively. The \texttt{humanoidmaze} counts apply to both dataset types.}
\label{tab:benchmark-details}
\end{table}

\textbf{Evaluations.}
We evaluate all the methods on the five prescribed \texttt{OGBench} tasks, with 50 episodes per task and the benchmark's success criterion and episode limits, following the standard evaluation protocol. Success rates (\%) are averaged across tasks for each training seed, and we report the mean and standard deviation over 10 seeds. Collision rates (\%) measure the percentage of evaluation steps with at least one \texttt{MuJoCo} contact between the agent and the domain boundary. Exact training budgets and hyperparameters are provided below.

\begin{figure}[t]
  \centering
  \setlength{\fboxsep}{0pt}\setlength{\fboxrule}{0.6pt}%
  \fbox{\includegraphics[width=0.24\linewidth]{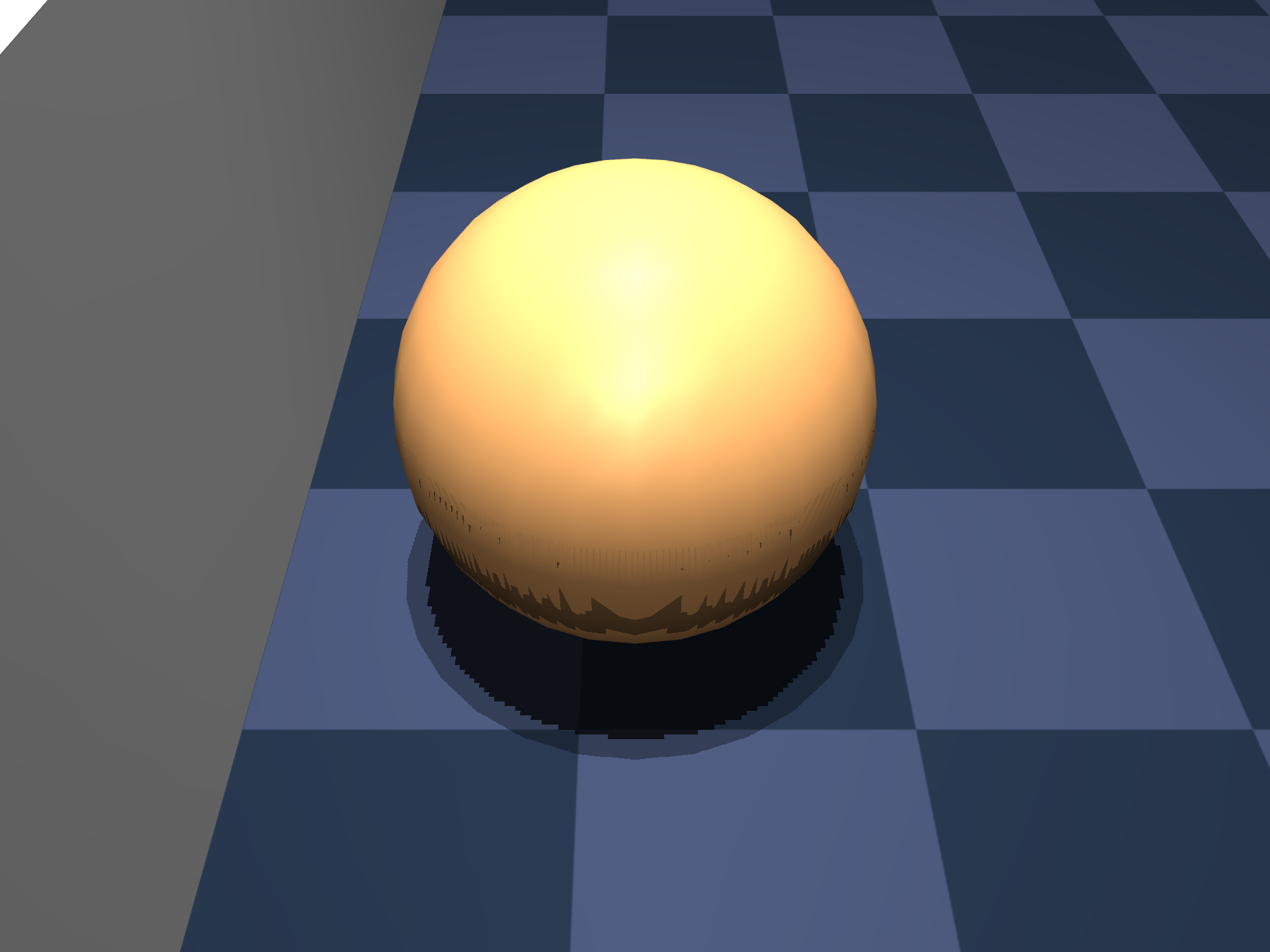}}\hspace{0.02\linewidth}%
  \fbox{\includegraphics[width=0.24\linewidth]{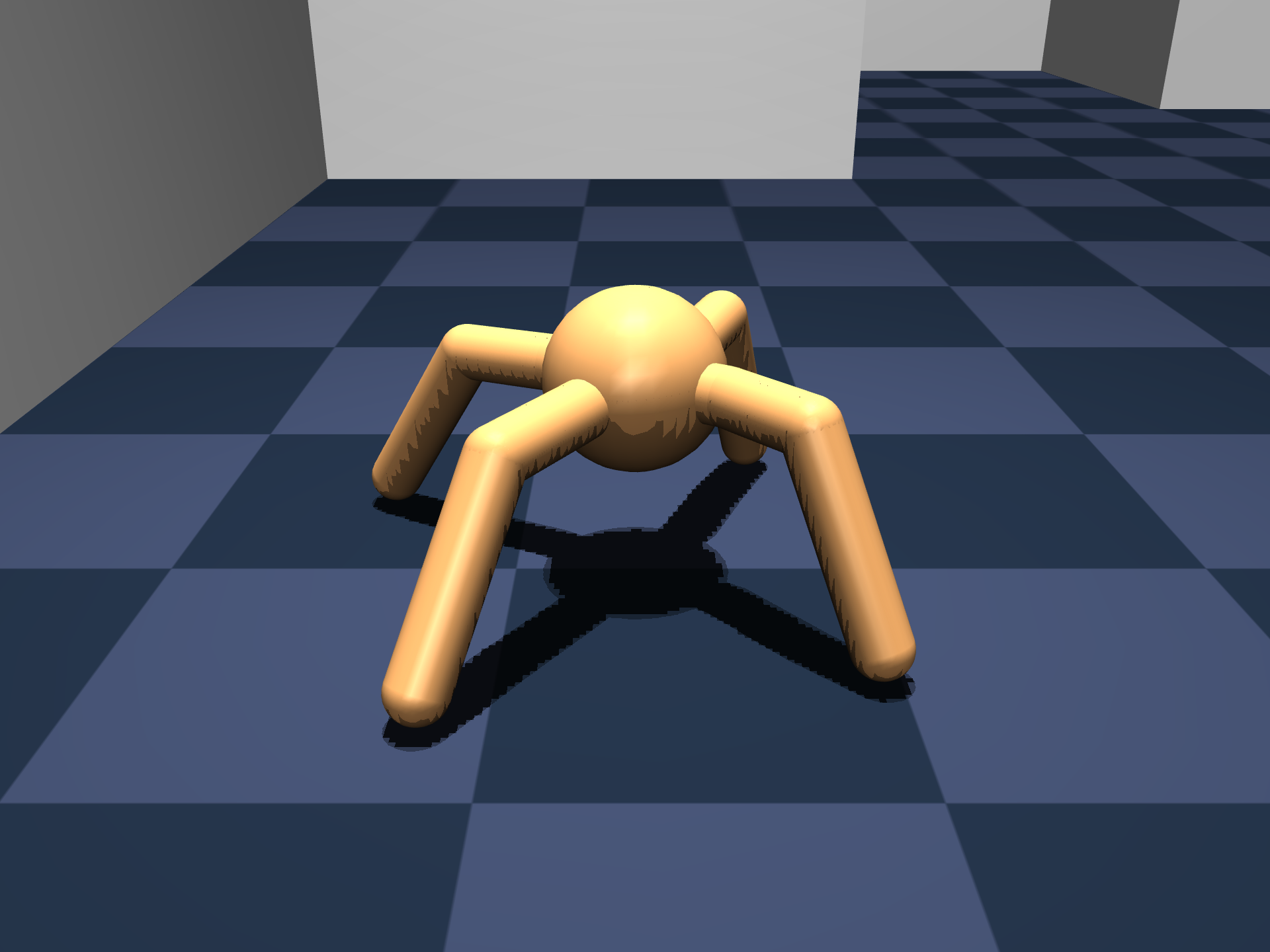}}\hspace{0.02\linewidth}%
  \fbox{\includegraphics[width=0.24\linewidth]{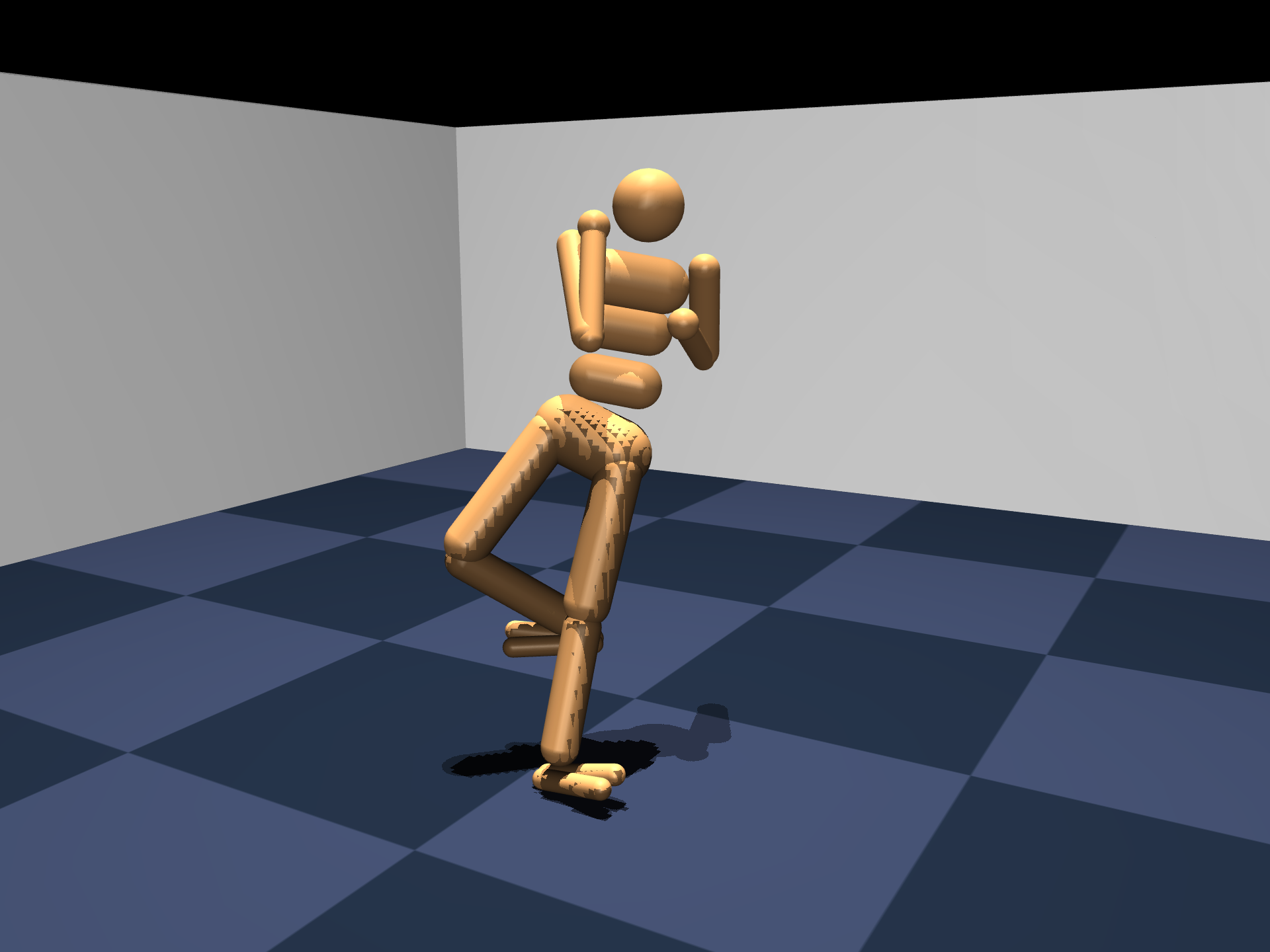}}
  \caption{The \texttt{pointmaze}, \texttt{antmaze}, and \texttt{humanoidmaze} agents of \texttt{OGBench}.}
  \label{fig:agents}
\end{figure}

\subsection{Reachability and Boundary Estimation}\label{app:c-est}

\textbf{State representations.}
For all three environment families, the critic $d_\theta$ uses the planar physical position $x_s=(x,y) \in \mathbb R^2$ of the full $\mathcal S$, with goals represented by the corresponding coordinates $x_g$, as we mention in Section~\ref{sec:offline-inst}. To estimate $\mathcal R_\tau^{\mathrm{dyn}}(s)$, we use reduced states $r_s$ containing these coordinates and additional dynamical variables. For \texttt{pointmaze}, $r_s=(x,y)\in\mathbb R^2$. For \texttt{antmaze}, $r_s=(x,y,v_x,v_y,\cos\psi,\sin\psi)\in\mathbb R^6$, where $(v_x,v_y)$ are planar translational velocities and $\psi$ is the heading angle obtained from the body's orientation quaternion. For \texttt{humanoidmaze}, $r_s=(x,y,v_x,v_y,o_1,o_2,o_3)\in\mathbb R^7$, where $(v_x,v_y)$ are root planar velocities and $(o_1,o_2,o_3)$ encode the world vertical direction in the torso frame. We propagate the dynamics in these reduced states and retain the $(x,y)$ displacements for the critic constraint. These coordinates and dynamical variables are selected using only offline observations, without additional environment interaction. We also use $x_s$ directly as the waypoint representation, avoiding the separate learned waypoint embedding used in HIQL and EikHiQRL \citep{park2023hiql,giammarino2026goal}.

\textbf{Estimating the local constraint geometry.}
To estimate $\mathcal R_\tau^{\mathrm{dyn}}(s)$,, for each environment family and dataset type of \texttt{OGBench}, we fit one reduced dynamics model a priori using \emph{only} the training dataset $\mathcal{D}$ and fix it throughout the training of the critic $d_\theta$. We form one-step pairs $(r_s,a,r_{s'})$ within each recorded trajectory $(s,a,s') \in \mathcal D$, excluding pairs that cross episode cutoffs. Then, we estimate the affine model parameters by minimizing ordinary least squares (OLS) \citep{ljung1999system,verhaegen2007filtering}:
\[
    (A_r,B_r,b_r)
    \in\arg\min_{A,B,b}
    \sum_{(s,a,s')}
    \|r_{s'}-Ar_s-Ba-b\|_2^2.
\]
For residuals $e_i=r_{s_i'}-A_r r_{s_i}-B_r a_i-b_r$, we then compute
\[
    C=\frac{1}{N}\sum_{i=1}^N e_i e_i^\top
      =Q\Lambda Q^\top,
    \qquad
    G_w=2\Lambda^{1/2}Q^\top.
\]
The empirical error envelope is given by $\{G_w^\top\xi:\|\xi\|_\infty\leq1\}$, with zonotope generators spanning two empirical standard deviations along each principal residual direction. For \texttt{pointmaze}, we fit transitions unaffected by boundary contact. At horizon $\tau=k\Delta t$, the propagated center in the reduced state space is obtained as
\[
    A_r^k r_s+\sum_{j=0}^{k-1}A_r^j b_r,
\]
and the generators comprise the action blocks
$A_r^{k-1}B_r,\ldots,B_r$ and error blocks
$A_r^{k-1}G_w^\top,\ldots,G_w^\top$, with coefficients ranging independently over $[-1,1]$ \citep{girard2005reachability,althoff2021set}.
We retain the planar-position entries of the center and corresponding generator rows, then subtract $x_s$ from the center to obtain $c_s$ and $G_k$ in \eqref{eq:estimated-rk}. The generator matrix is shared across states for each fitted model and horizon, while $c_s$ depends on the current reduced state. These quantities are precomputed where possible, and the critic constraint is evaluated through the closed-form support in \eqref{eq:zonotope-support}.

To estimate boundary geometry and outward unit normals, we use grid resolutions of at most $0.5$ and one or two morphological closing iterations with a four-neighbor structuring element for each layout. We smooth the signed distance field with a Gaussian of standard deviation one grid cell, compute gradients by central differences, and exclude points with $\Vert\nabla\Phi\Vert_2\leq10^{-2}$.

To determine the horizon $\tau:=k\Delta t$ before critic training, we heuristically compare the control-generated part of $G_k$ with the uncontrolled displacement $c_s$, both obtained from the dynamics fitted to offline data, drawing on set-based reachability analysis \citep{blanchini2008set,althoff2021set,girard2005reachability}. For the action-generating columns of $G_k^{\mathrm{ctrl}}$ and the center $c_s$ (evaluated with zero initial position) of the fitted reachable set above, we compute the ratio as
\[
    \chi_\tau := \frac{\min_{\|v\|_2=1}\|(G_k^{\mathrm{ctrl}})^\top v\|_1}
    {\operatorname{Quan}_{s\sim\mathcal D}\|c_s\|_2},
\]
Quantile is taken at 0.9, and we use $\chi_\tau\in(0.5,1.5)$ as a heuristic range of comparable control and drift magnitudes. With this, we use $k=1$ in \texttt{pointmaze} (zero drift), $k=12$ in \texttt{antmaze} and $k=48$ in \texttt{humanoidmaze}. These horizons are fixed across layouts and dataset types within each environment family before the critic training.

\subsection{Training and Inference for Critic and Policy}

\textbf{Learning the \ours{} critic.}
To train our critic $d_\theta$, we implement \eqref{eq:normalized-objective} using soft constraint penalties. Let $\Delta_\theta(s,g)$ denote the predicted descent on the left-hand side of \eqref{eq:zonotope-support}. Our \ours{} objective is
\begin{equation}
\begin{aligned}
    \mathcal L_{\ours{}}(\theta)
    :={}&-\mathbb E_{s,g}\left[\phi(d_\theta(s,g))\right]\\
    &+w_{\mathrm{in}}\,\mathbb E_{s,g}\left[(\Delta_\theta(s,g)-1)^2\right]\\
    &+w_b\,\mathbb E_b\left[
    \operatorname{ReLU}\left(
    -\widehat\nu(x_s)^\top\nabla_{x_s}d_\theta(s,g)
    \right)^2\right],
\end{aligned}
\end{equation}
where $\phi(z)$ is the monotone distance transform from QRL \citep[Eq.~(40)]{wang2023optimal}. The parameters $w_{\mathrm{in}},w_b>0$ control the interior and boundary penalty weights, and $\mathbb E_b$ averages over sampled boundary $(s,g)$ pairs. We follow EikHiQRL's $(s,g)$ sampling for \texttt{antmaze} and \texttt{humanoidmaze} \citep{giammarino2026goal}, while using random dataset goals for the \texttt{pointmaze} critic. Across all settings, the boundary penalty uses sampled states within the estimated boundary band. We use exact zonotopic support from \eqref{eq:zonotope-support}, except in \texttt{pointmaze}, where we use a smooth log-sum-exp approximation and anneal its temperature from $1$ to $0.05$ during training, which converges to the exact support as the temperature tends to zero.

\textbf{Policy training.}
We train policies by advantage-weighted regression (AWR) \citep{peng2019advantage}, using HIQL's simplified advantage estimates \citep{park2023hiql}. For $k=1$, the policy $\pi_\psi(a\mid x_s,x_g)$ predicts primitive actions and is trained on recorded transitions:
\[
\begin{aligned}
    A(s,s',g)&=d_\theta(s,g)-d_\theta(s',g),\\
    \mathcal L_\pi(\psi)
    &=-\mathbb E_{s,a,s',g}\left[
    \min\{\exp(\beta A(s,s',g)),C\}
    \log\pi_\psi(a\mid x_s,x_g)\right].
\end{aligned}
\]

For $k>1$, $\pi_\psi(w\mid x_s,x_g)$ predicts waypoints through offsets from $x_s$. Its targets are $k$-step hindsight waypoints $w=x_{s_{t+k}}$, capped at the trajectory end or the sampled future goal's time, whichever comes first:
\[
\begin{aligned}
    A^h(s,w,g)&=d_\theta(x_s,x_g)-d_\theta(w,x_g),\\
    \mathcal L_{\pi_h}(\psi)
    &=-\mathbb E_{s,w,g}\left[
    \min\{\exp(\beta_h A^h(s,w,g)),C\}
    \log\pi_\psi(w\mid x_s,x_g)\right].
\end{aligned}
\]
A separate low-level value function takes the full state $s$ and coordinate waypoint $w$. Its two heads use action-free, goal-conditioned expectile TD learning with Polyak target networks, following HIQL and EikHiQRL \citep{kostrikov2021offline,park2023hiql,giammarino2026goal}, without a learned waypoint embedding or an additional Eikonal penalty on the low-level value. Writing $V_\eta$ for the mean of these heads, we train the action policy on recorded actions and $k$-step waypoints truncated at the trajectory end:
\[
\begin{aligned}
    A^\ell(s,s',w)&=V_\eta(s',w)-V_\eta(s,w),\\
    \mathcal L_{\pi_\ell}
    &=-\mathbb E_{s,a,s',w}\left[
    \min\{\exp(\beta_\ell A^\ell(s,s',w)),C\}
    \log\pi_\ell(a\mid s,w)\right].
\end{aligned}
\]
All advantage weights are detached during policy updates, with $C=100$. Following EikHiQRL, policy goals are uniformly sampled future states for \texttt{navigate}, an equal mixture of future and random dataset states for \texttt{stitch}, and random dataset states for \texttt{explore}. Low-level value training uses the critic-goal sampling described above. These policy-training components retain the corresponding baseline settings; we use coordinate waypoints directly and omit the additional Eikonal penalty on the low-level value function.

\begin{algorithm}[h]
\caption{\ours{} policy execution}
\label{alg:inference}
\begin{algorithmic}[1]
\Require Initial state $s_0 \in \mathcal S$, goal $g \in \mathcal S$, horizon $k$, episode limit $T_{\max}$
\Require Trained actor $\pi_\psi$ for \ours{} critic $d_\theta$ and, for $k>1$, low-level actor $\pi_\ell$
\State Extract goal coordinates $x_g$
\For{$t=0,\ldots,T_{\max}-1$}
    \State Extract $x_{s_t}$ from $s_t$
    \If{$k=1$}
        \State $a_t\gets\operatorname{clip}(\mu_\psi(x_{s_t},x_g),-1,1)$
        \Comment{Mean primitive action}
    \Else
        \State $w_t\gets x_{s_t}+\operatorname{clip}(\mu_\psi(x_{s_t},x_g),-10,10)$
        \Comment{Mean offset to waypoint}
        \State $a_t\gets\operatorname{clip}(\mu_\ell(s_t,w_t),-1,1)$
        \Comment{Waypoint-conditioned action}
    \EndIf
    \State Execute $a_t$ and observe $s_{t+1}$
    \If{the goal is reached or the episode terminates or is truncated}
        \State \textbf{break}
    \EndIf
\EndFor
\end{algorithmic}
\end{algorithm}

\textbf{Inference.}
For $k=1$, the trained policy $\pi_\psi(a\mid x_s,x_g)$ directly predicts primitive actions. For $k>1$, the waypoint policy $\pi_\psi(w\mid x_s,x_g)$ predicts an offset from $x_s$, and the low-level policy $\pi_\ell(a\mid s,w)$ maps the full observation and resulting waypoint to an action. We use mean outputs, denoted by $\mu_\psi$ and $\mu_\ell$, with componentwise clipping. Waypoints are recomputed at every environment step. Evaluation ends upon goal achievement, environment termination, or the benchmark's episode time limit. Refer to Algorithm~\ref{alg:inference}.

\subsection{Baselines and their Implementation}
We compare against quasimetric methods (QRL, EikHiQRL) and non-quasimetric methods (GCIQL, CRL, HIQL, SAW, OTA, GAS). GCIQL~\citep{kostrikov2021offline} learns a goal-conditioned value by expectile regression. HIQL~\citep{park2023hiql} learns an action-free value and extracts subgoal and low-level policies from it. SAW~\citep{zhou2026flattening} trains a single flat policy whose actions for distant goals are regressed toward those of a short-horizon policy. CRL~\citep{eysenbach2022contrastive} trains the critic as a contrastive classifier over future states. OTA~\citep{ahn2026option} learns the high-level value over temporally abstracted transitions. GAS~\citep{baek2025graph} plans over a graph of key states in a learned temporal-distance space. Among the qausimetric methods, QRL~\citep{wang2023optimal} learns an IQE critic with transition-cost constraints and EikHiQRL~\citep{giammarino2026goal} adds an Eikonal gradient penalty and a hierarchical actor of HIQL.

All methods are reimplemented in \texttt{PyTorch} that follow each method's released \texttt{JAX} code, sharing the same training loop, evaluation protocol, and seeds. All baselines including our method are written in \texttt{PyTorch} and are compiled before running. Hyperparameters follow the authors' per-environment settings when available and are otherwise tuned, all following the same sweep protocol.

\subsection{Training and Hyperparameters}

All networks use three hidden layers of width $512$ with GELU activations. Critics additionally use LayerNorm. The quasimetric critic uses IQE \citep{wang2022improved} with a $512$-dimensional latent representation, organized into $64$ components of size $8$. Our QRL and EikHiQRL baselines use the same IQE critic architecture. Actors parameterize diagonal Gaussian distributions with fixed unit standard deviation, and the low-level value function has two heads. We optimize the networks with \texttt{Adam} and batch size $1024$. The inherited policy-training settings follow EikHiQRL: the low-level discount is $\gamma = 0.995$, the expectile is $0.7$, and the Polyak update coefficient is $0.005$. AWR coefficients are $\alpha = 3$, except for \texttt{antmaze-explore}, which uses $\alpha = 10$. Our additional hyper-parameters are \ours{}'s critic loss weights, $w_{\mathrm{in}}$ and $w_b$. the resulting values of $w_{\mathrm{in}}$ and $w_B$ are in  Table~\ref{tab:training-settings}. All methods receive the same number of training updates within each setting: $10^5$ training steps for \texttt{pointmaze} and \texttt{antmaze}, except $10^6$ steps for \texttt{antmaze-explore}, and $10^6$ training steps for \texttt{humanoind} and \texttt{antsoccer}. 

\begin{table}[h]
\centering
\small
\setlength{\tabcolsep}{4pt}
\begin{tabular}{llcc}
\toprule
Environment & Dataset & $w_{\mathrm{in}}$ & $w_b$ \\
\midrule
\texttt{pointmaze} & \texttt{navigate}
& $(0.5,1,1)\times10^2$ & $(1,1,5)\times10^1$ \\
& \texttt{stitch}
& $(2.5,2,2)\times10^2$ & $(0.5,1,1.5)\times10^1$ \\
\texttt{antmaze} & \texttt{navigate}
& $(3,3,3)$ & $(1,1,3)\times10^2$ \\
& \texttt{stitch}
& $(2.5,2.5,5)\times10^{-1}$ & $(1,1,0.5)\times10^2$ \\
& \texttt{explore}
& $(2.5,2.5)\times10^{-1}$ & $(1,0.5)\times10^2$ \\
\texttt{humanoidmaze} & \texttt{navigate}
& $(5,5,5)$ & $(3,3,5)\times10^2$ \\
& \texttt{stitch}
& $(0.5,0.5,1)$ & $(1,1,1)\times10^2$ \\
\bottomrule
\end{tabular}
\caption{Selected penalty weights $w_{\mathrm{in}}$ and $w_b$, ordered as (\texttt{medium}, \texttt{large}, \texttt{giant}). For \texttt{explore}, entries correspond to (\texttt{medium}, \texttt{large}).}
\label{tab:training-settings}
\end{table}
\section{Additional Experiments Details and Results}\label{app:add}

We include additional experimental results and details in this section. Specifically, we provide the full result tables supplementing the main text, experimental details and additional results for \texttt{carmaze} and \texttt{antsoccer}, and runtime comparison between \ours{} and our reimplemented \texttt{PyTorch} baselines. 

\begin{table}[h]
\centering
\scriptsize
\setlength{\tabcolsep}{4.5pt}
\renewcommand{\arraystretch}{1.1}
\begin{tabular}{ll l llllll @{\hspace{10pt}} lll}
\toprule
\multirow{2}{*}{\textbf{Env}} & \multirow{2}{*}{\textbf{Data}} & \multirow{2}{*}{\textbf{Size}}
& \multicolumn{6}{c}{\textbf{Other GCRL}} & \multicolumn{3}{c}{\textbf{Quasimetric}} \\
\cmidrule(lr){4-9}\cmidrule(lr){10-12}
& & & \multicolumn{1}{c}{CRL} & \multicolumn{1}{c}{SAW} & \multicolumn{1}{c}{GCIQL} & \multicolumn{1}{c}{HIQL} & \multicolumn{1}{c}{OTA} & \multicolumn{1}{c}{GAS}
& \multicolumn{1}{c}{QRL} & \multicolumn{1}{c}{EikHiQRL} & \multicolumn{1}{c}{\textbf{\ours}} \\
\midrule

\multirow{6}{*}{\rotatebox[origin=c]{90}{\texttt{carmaze}}}
& \multirow{3}{*}{\texttt{nav}}
& \texttt{medium}   & $50 \pm 5$ & $52 \pm 7$ & $28 \pm 4$ & $53 \pm 5$ & $63 \pm 7$  & $81 \pm 5$ & $36 \pm 6$ & $72 \pm 8$ & \best{$90 \pm 3$} \\
& & \texttt{large}  & $13 \pm 4$ & $2 \pm 2$  & $4 \pm 1$  & $19 \pm 8$ & $32 \pm 14$ & $67 \pm 7$ & $4 \pm 4$ & $45 \pm 15$ & \best{$81 \pm 12$} \\
& & \texttt{giant}  & $0 \pm 0$  & $0 \pm 0$  & $0 \pm 0$  & $1 \pm 1$  & $3 \pm 3$   & $40 \pm 8$ & $1 \pm 1$ & $13 \pm 9$  & \best{$60 \pm 14$} \\
\cmidrule(lr){2-12}
& \multirow{3}{*}{\texttt{sti}}
& \texttt{medium}   & $28 \pm 5$ & $27 \pm 11$ & $10 \pm 2$ & $23 \pm 7$ & $60 \pm 6$  & $74 \pm 6$  & $30 \pm 6$ & $69 \pm 9$ & \best{$82 \pm 6$} \\
& & \texttt{large}  & $0 \pm 0$  & $1 \pm 1$   & $0 \pm 1$  & $0 \pm 1$  & $14 \pm 6$  & $64 \pm 10$ & $1 \pm 1$ & $26 \pm 9$ & \best{$78 \pm 12$} \\
& & \texttt{giant}  & $0 \pm 0$  & $0 \pm 0$   & $0 \pm 0$  & $0 \pm 0$  & $0 \pm 0$   & $31 \pm 8$  & $0 \pm 0$ & $5 \pm 7$  & \best{$45 \pm 17$} \\
\bottomrule
\end{tabular}
\caption{\textbf{Success rates (\%)} on \texttt{carmaze} at the final checkpoint (mean $\pm$ std over 10 seeds; 50 episodes per goal). \texttt{nav} and \texttt{sti} denote \texttt{navigate} and \texttt{stitch}. Means at least $95\%$ of the best are \textbf{bold}; the best is \underline{underlined}.}
\label{tab:carmaze}
\end{table}

\textbf{Additional details on \texttt{carmaze}.} We elaborate on the \texttt{carmaze} environment and datasets discussed and evaluated in Section~\ref{sec:exp} and Table~\ref{tab:carmaze}. We integrate the \texttt{Safety-Gymnasium Car} robot \citep{ji2023safety}, simulated in \texttt{MuJoCo} \citep{todorov2012mujoco}, into \texttt{OGBench}'s official \texttt{medium}, \texttt{large}, and \texttt{giant} layouts \citep{park2025ogbench}. The 22-D state space $\mathcal S$ consists of the 3-D chassis position, two 4-D orientation quaternions for the chassis and caster, and 11 generalized velocities. The two unbounded wheel angles are omitted, while their angular velocities are retained. Its 2-D actions are in $\mathcal A := [-1,1]^2$ and control the two wheel motors with torque scale $0.02$. We set wheel damping to $0.0005$ and use a $0.1$\,s control interval, comprising 25 simulator steps of $0.004$\,s. Following \texttt{OGBench}'s collection pipeline, we train an SAC locomotion expert for $400$k steps, then generate trajectories using waypoint guidance with Gaussian action noise. Each \texttt{navigate} dataset contains 1,000 trajectories, with approximately $0.6$M, $1$M, and $1.6$M samples for \texttt{medium}, \texttt{large}, and \texttt{giant}, respectively. Each \texttt{stitch} dataset contains 5,000 short trajectories, totaling approximately $1$M samples per layout.

For estimating the reachable set and estimate boundary normals, \ours{} uses critic goal-coordinates $x_s=(x,y)\in\mathbb{R}^2$ and reduced states $r_s=(x,y,v_x,v_y,\cos\psi,\sin\psi)\in\mathbb{R}^6$, as for \texttt{antmaze} (see Appendix~\ref{app:exp}). We also fix $k=16$ across all layouts and dataset types, consistent with the horizon-selection rationale in Appendix~\ref{app:c-est}. The remaining architecture, estimation, and policy-training details follow Appendix~\ref{app:exp}. Training uses $10^5$ training updates; GAS uses $10^5$ updates in each of its temporal-distance and policy-training stages, with graph construction between them. We tune method-specific hyperparameters for \ours{} and all baselines with equal effort. We initialize experiments with these selected configurations, additionally sweeping QRL, HIQL, OTA, and EikHiQRL and conducting a local search for GAS. For \ours{} specifically, we sweep $w_{\mathrm{in}}\in\{0.25,1,5,100\}$ and $w_b\in\{1,10,100\}$ and use $(w_{\mathrm{in}},w_b) = (5,100)$ for all \texttt{navigate} layouts, $(1,10)$ for \texttt{medium}/\texttt{large-stitch}, and $(0.25,1)$ for \texttt{giant-stitch}. Final evaluation follows Appendix~\ref{app:exp}. The full comparison in Table~\ref{tab:carmaze} supports the main-text results, with \ours{} achieving the highest mean success in all six settings.

\begin{table}[t]
\centering
\scriptsize
\setlength{\tabcolsep}{4.5pt}
\renewcommand{\arraystretch}{1.1}
\begin{tabular}{ll l llllll @{\hspace{10pt}} lll}
\toprule
\multirow{2}{*}{\textbf{Env}} & \multirow{2}{*}{\textbf{Data}} & \multirow{2}{*}{\textbf{Size}}
& \multicolumn{6}{c}{\textbf{Other GCRL}} & \multicolumn{3}{c}{\textbf{Quasimetric}} \\
\cmidrule(lr){4-9}\cmidrule(lr){10-12}
& & & \multicolumn{1}{c}{CRL} & \multicolumn{1}{c}{SAW} & \multicolumn{1}{c}{GCIQL} & \multicolumn{1}{c}{HIQL} & \multicolumn{1}{c}{OTA} & \multicolumn{1}{c}{GAS}
& \multicolumn{1}{c}{QRL} & \multicolumn{1}{c}{EikHiQRL} & \multicolumn{1}{c}{\textbf{\ours}} \\
\midrule

\multirow{6}{*}{\rotatebox[origin=c]{90}{\texttt{pointmaze}}}
& \multirow{3}{*}{\texttt{nav}}
& \texttt{medium}  & $34 \pm 5$ & $1 \pm 1$   & $56 \pm 2$ & $3 \pm 2$   & $2 \pm 1$  & $1 \pm 1$ & $33 \pm 7$ & $2 \pm 4$ & $7 \pm 6$ \\
& & \texttt{large} & $49 \pm 9$ & $27 \pm 12$ & $82 \pm 2$ & $32 \pm 7$  & $6 \pm 2$  & $0 \pm 0$ & $40 \pm 7$ & $0 \pm 1$ & $5 \pm 4$ \\
& & \texttt{giant} & $58 \pm 8$ & $50 \pm 13$ & $93 \pm 7$ & $59 \pm 12$ & $14 \pm 5$ & $2 \pm 2$ & $43 \pm 9$ & $9 \pm 1$ & $10 \pm 7$ \\
\cmidrule(l){2-12}
& \multirow{3}{*}{\texttt{sti}}
& \texttt{medium}  & $72 \pm 12$ & $31 \pm 8$  & $78 \pm 3$ & $33 \pm 8$  & $13 \pm 3$  & $6 \pm 6$ & $32 \pm 7$  & $2 \pm 4$   & $9 \pm 7$ \\
& & \texttt{large} & $65 \pm 13$ & $55 \pm 8$  & $71 \pm 3$ & $68 \pm 10$ & $20 \pm 10$ & $3 \pm 5$ & $37 \pm 4$  & $1 \pm 1$   & $7 \pm 5$ \\
& & \texttt{giant} & $72 \pm 12$ & $57 \pm 15$ & $81 \pm 8$ & $76 \pm 10$ & $57 \pm 14$ & $0 \pm 0$ & $47 \pm 12$ & $14 \pm 15$ & $22 \pm 8$ \\
\midrule
\multirow{8}{*}{\rotatebox[origin=c]{90}{\texttt{antmaze}}}
& \multirow{3}{*}{\texttt{nav}}
& \texttt{medium}  & $8 \pm 1$  & $4 \pm 1$  & $11 \pm 1$ & $4 \pm 1$  & $3 \pm 0$ & $2 \pm 0$ & $5 \pm 1$  & $3 \pm 0$ & $1 \pm 0$ \\
& & \texttt{large} & $9 \pm 1$  & $7 \pm 1$  & $11 \pm 1$ & $6 \pm 1$  & $3 \pm 0$ & $3 \pm 1$ & $7 \pm 2$  & $5 \pm 0$ & $2 \pm 0$ \\
& & \texttt{giant} & $14 \pm 1$ & $15 \pm 1$ & $14 \pm 1$ & $14 \pm 1$ & $9 \pm 1$ & $8 \pm 3$ & $14 \pm 1$ & $6 \pm 1$ & $4 \pm 1$ \\
\cmidrule(l){2-12}
& \multirow{3}{*}{\texttt{sti}}
& \texttt{medium}  & $15 \pm 1$ & $9 \pm 1$ & $12 \pm 1$ & $9 \pm 1$ & $3 \pm 0$ & $2 \pm 0$ & $6 \pm 1$ & $4 \pm 0$ & $1 \pm 0$ \\
& & \texttt{large} & $17 \pm 1$ & $15 \pm 1$ & $11 \pm 1$ & $15 \pm 1$ & $8 \pm 1$ & $3 \pm 1$ & $10 \pm 1$ & $5 \pm 1$ & $1 \pm 0$ \\
& & \texttt{giant} & $16 \pm 2$ & $19 \pm 1$ & $13 \pm 1$ & $15 \pm 1$ & $13 \pm 1$ & $5 \pm 1$ & $16 \pm 1$ & $7 \pm 1$ & $4 \pm 1$ \\
\cmidrule(l){2-12}
& \multirow{2}{*}{\texttt{exp}}
& \texttt{medium}  & $8 \pm 3$ & $17 \pm 2$ & $23 \pm 1$ & $15 \pm 2$ & $2 \pm 1$ & $4 \pm 1$ & $21 \pm 2$ & $13 \pm 3$ & $1 \pm 1$ \\
& & \texttt{large} & $6 \pm 2$ & $18 \pm 1$ & $21 \pm 1$ & $17 \pm 2$ & $5 \pm 2$ & $5 \pm 1$ & $19 \pm 1$ & $11 \pm 2$ & $1 \pm 1$ \\
\midrule
\multirow{6}{*}{\rotatebox[origin=c]{90}{\texttt{humanoid}}}
& \multirow{3}{*}{\texttt{nav}}
& \texttt{medium}  & $21 \pm 1$ & $10 \pm 1$ & $25 \pm 1$ & $9 \pm 1$  & $8 \pm 1$ & $7 \pm 2$  & $16 \pm 4$ & $9 \pm 1$  & $3 \pm 0$ \\
& & \texttt{large} & $23 \pm 2$ & $13 \pm 1$ & $24 \pm 1$ & $15 \pm 2$ & $9 \pm 1$ & $9 \pm 1$  & $18 \pm 3$ & $13 \pm 2$ & $10 \pm 3$ \\
& & \texttt{giant} & $33 \pm 4$ & $17 \pm 2$ & $28 \pm 2$ & $23 \pm 2$ & $9 \pm 1$ & $11 \pm 1$ & $25 \pm 2$ & $11 \pm 1$ & $7 \pm 1$ \\
\cmidrule(l){2-12}
& \multirow{3}{*}{\texttt{sti}}
& \texttt{medium}  & $22 \pm 1$ & $9 \pm 1$  & $21 \pm 1$ & $8 \pm 1$  & $7 \pm 0$  & $6 \pm 1$  & $15 \pm 2$ & $10 \pm 1$ & $4 \pm 1$ \\
& & \texttt{large} & $26 \pm 2$ & $12 \pm 1$ & $23 \pm 2$ & $13 \pm 1$ & $10 \pm 1$ & $9 \pm 1$  & $18 \pm 2$ & $13 \pm 1$ & $6 \pm 0$ \\
& & \texttt{giant} & $32 \pm 3$ & $25 \pm 2$ & $26 \pm 2$ & $26 \pm 2$ & $12 \pm 1$ & $13 \pm 1$ & $27 \pm 2$ & $13 \pm 1$ & $7 \pm 1$ \\
\bottomrule
\end{tabular}
\caption{\textbf{Collision rate (\%)} ($\downarrow$): percentage of time steps with agent-boundary contact in \texttt{MuJoCo} (mean $\pm$ std over 10 seeds). \texttt{nav}, \texttt{sti}, \texttt{exp} are shorthands for \texttt{navigate}, \texttt{stitch}, \texttt{explore} datasets of \texttt{OGBench}, respectively.}
\label{tab:collision}
\end{table}

\textbf{Supplementary experimental results.}
In Table~\ref{tab:collision}, the full collision results are consistent with the main-text comparison on \texttt{antmaze} and \texttt{humanoidmaze}. \ours{} achieves lower mean collision rates than QRL and EikHiQRL across settings, and generally lower or comparable rates relative to OTA and GAS. In \texttt{pointmaze}, however, the point mass can touch corners or slide along walls while progressing toward the goal, so successful trajectories can still accumulate contact steps, which is consistent with the reported results in Table~\ref{tab:main}.

\begin{table}[t]
\centering
\small
\setlength{\tabcolsep}{5pt}
\begin{tabular}{l l ll @{\hspace{10pt}} lll}
\toprule
\multirow{2}{*}{\textbf{Env}} & \multirow{2}{*}{\textbf{Size}}
& \multicolumn{2}{c}{\textbf{Other GCRL}} & \multicolumn{3}{c}{\textbf{Quasimetric}} \\
\cmidrule(lr){3-4}\cmidrule(lr){5-7}
& & \multicolumn{1}{c}{CRL} & \multicolumn{1}{c}{HIQL}
& \multicolumn{1}{c}{QRL} & \multicolumn{1}{c}{EikHiQRL} & \multicolumn{1}{c}{\textbf{\ours}} \\
\midrule
\multirow{2}{*}{\texttt{antsoccer-nav}}
& \texttt{arena}  & $24 \pm 2$ & \second{$60 \pm 4$} & $10 \pm 3$ & \second{$61 \pm 5$} & \best{$65 \pm 3$} \\
& \texttt{medium} & $4 \pm 2$ & $13 \pm 3$ & $2 \pm 2$ & $13 \pm 1$ & \best{$17 \pm 3$} \\
\bottomrule
\end{tabular}
\caption{\textbf{Success rates (\%)} on \texttt{OGBench}'s \texttt{antsoccer-navigate} after $10^6$ training updates (mean $\pm$ std). For reference, baseline results are taken from \citet{giammarino2026goal} under the same training budget.}
\label{tab:antsoccer}
\end{table}

We additionally evaluate \ours{} on \texttt{OGBench}'s \texttt{antsoccer} tasks \citep{park2025ogbench}, where goal reaching requires moving an external object through contact. An ant must dribble a ball toward a target in the open \texttt{arena} or the \texttt{medium} maze. The 42-D state combines the 29-D ant state and 13-D ball state, with eight joint-torque actions in $\mathcal A=[-1,1]^8$. We use $10^6$ training updates with $k=24$, otherwise following Appendix~\ref{app:exp}. The critic coordinates $x_s\in\mathbb R^4$ contain the planar positions of the ant and ball. The reduced state $r_s\in\mathbb R^{12}$ contains the ant's planar position, planar velocity, and heading sine/cosine, together with the ball's 3-D position and linear velocity. Following the benchmark's goal observations, both goal-position blocks use the target location, while evaluation success depends only on the ball reaching within $0.5$\,m of the target. As in Appendix~\ref{app:c-est}, we fit the affine dynamics model, projecting the propagated reachable sets onto $x_s$. Table~\ref{tab:antsoccer} reports mean success rates of $65\%$ and $17\%$, compared with $61\%$ and $13\%$ for EikHiQRL using the published, best-of-10 evaluations at 1M \citep{giammarino2026goal}. These results extend our evaluation to goal-reaching tasks involving object interaction.

In Figure~\ref{fig:training-vs-success}, we also compare success rate against wall-clock training time on the \texttt{humanoidmaze} environment, using our compiled \texttt{PyTorch} implementations on a single \texttt{A100} GPU for $10^6$ training updates. Success rates are averaged across the \texttt{medium}, \texttt{large}, and \texttt{giant} layouts, separately for \texttt{navigate} and \texttt{stitch}. Overall, these results complement the main-text comparisons by showing competitive success together with practical training efficiency.

\begin{figure}[h]
\centering
\includegraphics[width=1.0\linewidth]{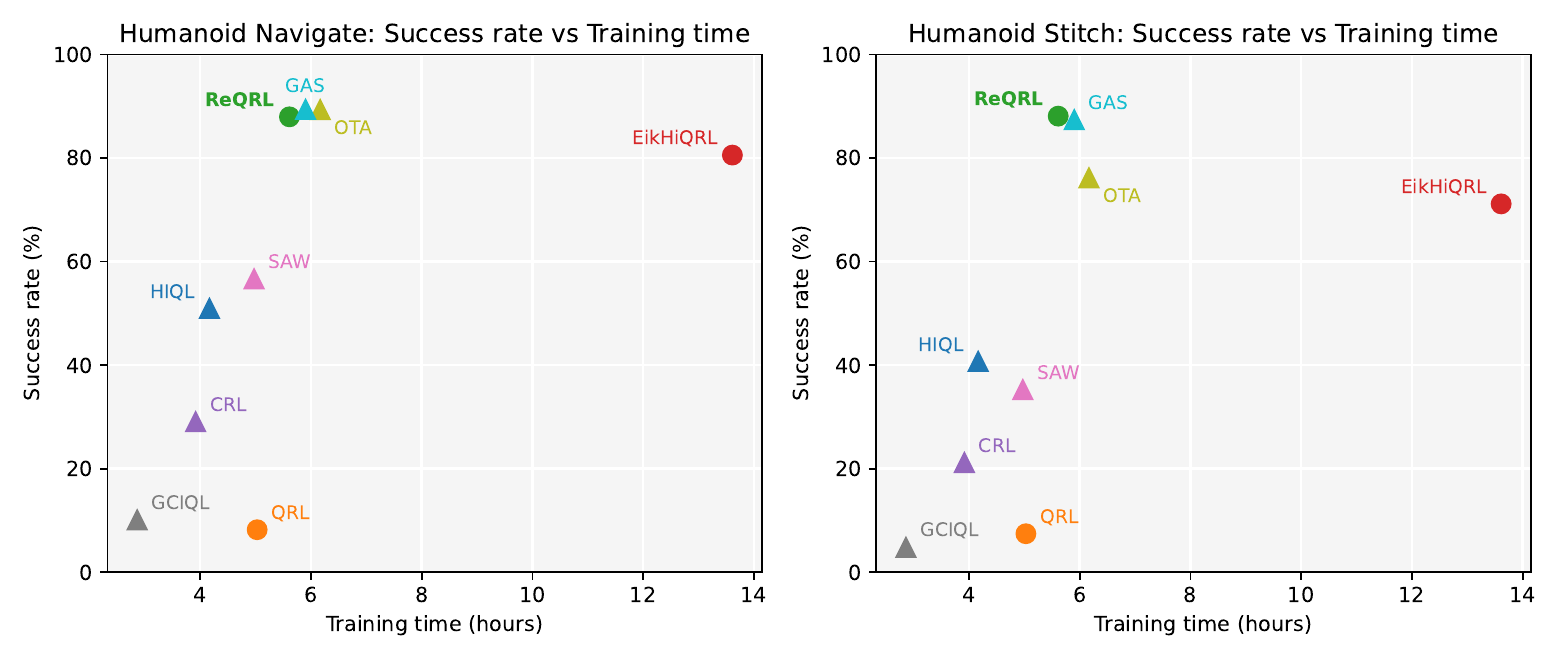}
\caption{Success rate averaged across maze sizes versus training time on \texttt{humanoidmaze} \texttt{navigate} and \texttt{stitch}. Wall-clock time covers $10^6$ training updates on one \texttt{A100} using compiled \texttt{PyTorch} implementations. Circles denote quasimetric methods; triangles denote other GCRL baselines.}

\label{fig:training-vs-success}
\end{figure}

\end{document}